\documentclass[10pt,twocolumn,letterpaper]{article}

\usepackage[pagenumbers]{cvpr} %

\usepackage[utf8]{inputenc} %
\usepackage[T1]{fontenc}    %
\usepackage{url}            %
\usepackage{booktabs}       %
\usepackage{amsfonts}       %
\usepackage{nicefrac}       %
\usepackage{microtype}      %
\usepackage[table]{xcolor}         %
\usepackage{amsmath}
\usepackage{amsthm}
\usepackage{graphicx}
\usepackage{subcaption}
\usepackage{multirow}
\usepackage{multicol}
\usepackage{bbold}
\usepackage{enumitem}
\usepackage{wrapfig}
\usepackage{tikz}
\usepackage{graphicx}

\newtheorem{theorem}{Theorem}[section]

\newtheorem{proposition}[theorem]{Proposition}
\newtheorem{corollary}[theorem]{Corollary}

\newtheorem{definition}{Definition}[section]

\definecolor{lavenderblue}{rgb}{0.9, 0.9, 1.0}

\newcommand{\siximagesquare}[6]{
        \begin{minipage}[c]{.664\linewidth}
            \includegraphics[width=\linewidth]{#1}
        \end{minipage}\hspace{0.004\linewidth}%
        \begin{minipage}[c]{.33\linewidth}
            \includegraphics[width=\linewidth]{#2}
            \vspace{-0.01\linewidth} 
            \includegraphics[width=\linewidth]{#3}
        \end{minipage}
        \begin{minipage}[c]{\linewidth}
            \vspace{-0.004\linewidth} 
            \includegraphics[width=.33\linewidth]{#4}\hspace{0.004\linewidth}%
            \includegraphics[width=.33\linewidth]{#5}\hspace{0.004\linewidth}%
            \includegraphics[width=.33\linewidth]{#6}
        \end{minipage}
}

\definecolor{cvprblue}{rgb}{0.21,0.49,0.74}
\definecolor{baselinecolor}{gray}{.92}
\newcommand{\basecell}[1]{\cellcolor{baselinecolor}}
\definecolor{vyellow}{rgb}{0.7490,0.5647,0}

\definecolor{cvprblue}{rgb}{0.21,0.49,0.74}
\usepackage[pagebackref,breaklinks,colorlinks,citecolor=cvprblue]{hyperref}

\def\paperID{281} %
\def\confName{CVPR}
\def\confYear{2024}

\title{\vspace{-2em}Don't drop your samples! Coherence-aware training benefits Conditional diffusion}

\author{Nicolas Dufour$^{1,2}$, Victor Besnier$^{3}$, Vicky Kalogeiton$^{2}$, David Picard$^{1}$ \\
\footnotesize{$^1$LIGM, École des Ponts, Univ Gustave Eiffel, CNRS, Marne-la-Vallée, France $\quad$
$^2$LIX, CNRS, École Polytechnique, IP Paris $\quad$
$^3$Valeo.ai, Prague}
}

\begin{document}
\twocolumn[{%
\renewcommand\twocolumn[1][]{#1}
\maketitle
\vspace{-2em}
\begin{tikzpicture}

\def\pagewidth{17} %
\pgfmathsetmacro{\fullsize}{0.25*\pagewidth}
\pgfmathsetmacro{\halfsize}{0.5*\fullsize}
\pgfmathsetmacro{\quartersize}{0.5*\halfsize}

\foreach \i in {0,1,2,3} {
    \node[inner sep=0pt, minimum size=\fullsize cm] at (\i*\fullsize cm, 0) {
        \includegraphics[width=\fullsize cm, height=\fullsize cm, clip]{images/teaser_v2/image_\i.png}
    };
}

\foreach \i in {0,1,2,3} {
    \foreach \j in {0,1} {
        \pgfmathsetmacro{\xpos}{\i*\fullsize +(\j-0.5)* \halfsize}
        \pgfmathsetmacro{\ypos}{-\fullsize + \halfsize/2}
        \node[inner sep=0pt, minimum size=\halfsize cm] at (\xpos cm, \ypos cm) {
            \includegraphics[width=\halfsize cm, height=\halfsize cm, clip]{images/teaser_v2/image_\the\numexpr4+2*\i+\j.png}
        };
    };
}
\foreach \i in {0,1,2,3} {
    \foreach \j in {0,1} {
        \pgfmathsetmacro{\xpos}{\i*\fullsize +(\j-0.5)* \halfsize}
        \pgfmathsetmacro{\ypos}{-\fullsize - \halfsize/2}
        \node[inner sep=0pt, minimum size=\halfsize cm] at (\xpos cm, \ypos cm) {
            \includegraphics[width=\halfsize cm, height=\halfsize cm, clip]{images/teaser_v2/image_\the\numexpr12+2*\i+\j.png}
        };
    };
}

\end{tikzpicture}
\vspace{-0.4em}\captionof{figure}{Images generated from our model, CAD. Our model showcase high visual quality, aesthetics and prompt following. \vspace{0.4em}}
\label{fig:teaser_v2}
}]

\begin{abstract}
    \noindent Conditional diffusion models are powerful generative models that can leverage various types of conditional information, such as class labels, segmentation masks, or text captions. However, in many real-world scenarios, conditional information may be noisy or unreliable due to human annotation errors or weak alignment. In this paper, we propose the Coherence-Aware Diffusion (CAD), a novel method that integrates coherence in conditional information into diffusion models, allowing them to learn from noisy annotations without discarding data. We assume that each data point has an associated coherence score that reflects the quality of the conditional information. We then condition the diffusion model on both the conditional information and the coherence score. In this way, the model learns to ignore or discount the conditioning when the coherence is low. We show that CAD is theoretically sound and empirically effective on various conditional generation tasks. Moreover, we show that leveraging coherence generates realistic and diverse samples that respect conditional information better than models trained on cleaned datasets where samples with low coherence have been discarded. Code and weights can be found \href{https://nicolas-dufour.github.io/cad.html}{here}.
\end{abstract}
    
\section{Introduction}

Conditional Diffusion models excel in image generation while affording greater user control over the generation process by integrating additional information~\cite{saharia2022photorealistic,balaji2022ediffi}. This extra data enables the model to guide the generated image towards a specific target, leading to improved various applications including high-quality text-to-image generation~\cite{rombach2022high}, as well as other modalities such as depth or human body pose~\cite{zhang2023adding}. Furthermore, the accessibility of open-source models like Stable Diffusion has democratized the use of this technology, already causing significant shifts in various domains such as design, art, and marketing.

\begin{figure*}[t]
  \begin{subfigure}[t]{0.65\linewidth}    
    \begin{tikzpicture}
        \foreach \i in {0,...,4} {
            \node at (2.2*\i,0) {\includegraphics[width=0.194\linewidth]{images/teaser/space_racoon/\i.png}};
            \node at (2.2*\i,-2.2) {\includegraphics[width=0.194\linewidth]{images/teaser/avocado_chair/\i.png}};
        }
       \draw[->, thick] (-0.6,-3.6) -- node[midway, fill=white] {Coherence score} (9,-3.6);
       \node at (-0.9,-3.6) {0};
        \node at (9.3,-3.6) {1};
    \end{tikzpicture}
    \caption{}
    \label{fig:clip_vs_conf_viz}
\end{subfigure}
\begin{subfigure}[t]{0.322\linewidth}
\includegraphics[width=\linewidth]{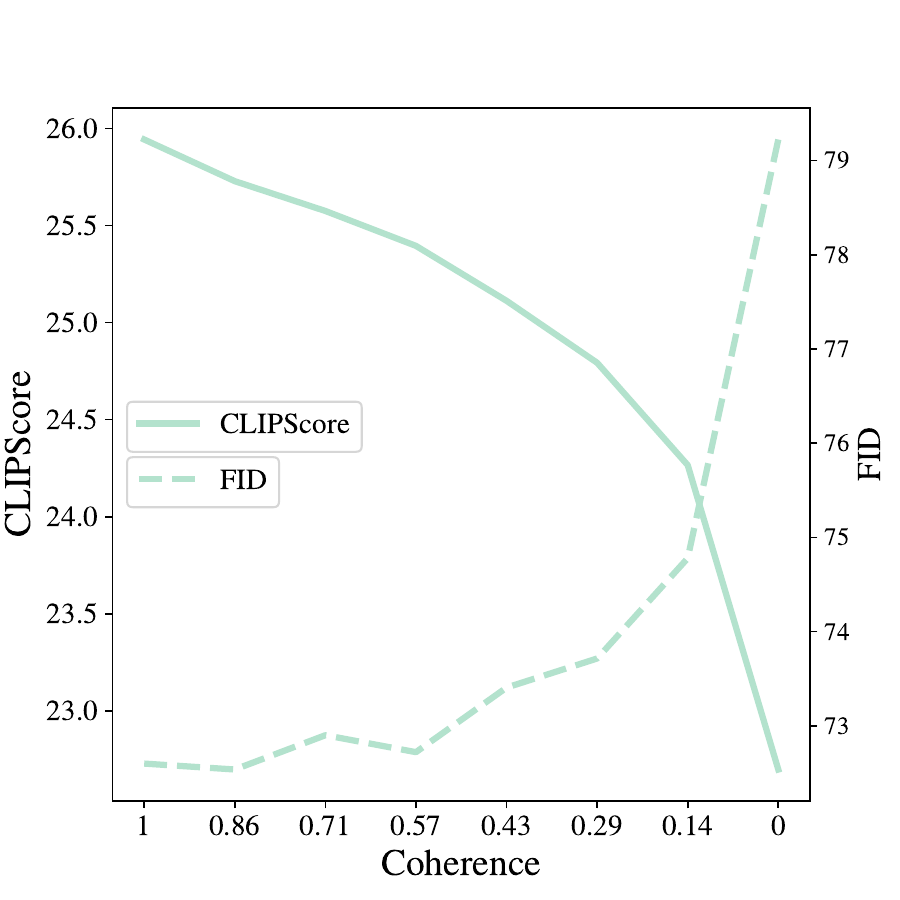}
    \caption{}
    \label{fig:clip_vs_confidence}
\end{subfigure}

    \caption{(a) Examples of images generated with the input \emph{coherence score} between the prompt and the target image. The score varies from 0 (no coherence) to 1 (maximum coherence). Higher coherence scores tend to generate images that adhere more effectively to the prompt. Top prompt: \emph{``a raccoon wearing an astronaut suit. The racoon is looking out of the window at a starry night; unreal engine, detailed, digital painting,cinematic,character design by pixar and hayao miyazaki, unreal 5, daz, hyperrealistic, octane render''}, bottom prompt: \emph{``An armchair in the shape of an avocado''} (b) Increasing the coherence from 0 to 1, CLIPScore increases and FID decreases.}
\end{figure*}

Training conditional diffusion models requires substantial volumes of paired data comprising the target image and its corresponding condition. In text-to-image generation, this pairing involves an image and a descriptive caption that characterizes both the content and the style of the image. Similarly, for class conditional generation, the pair consists of an image and its corresponding class label. Besides the technical challenges associated with the acquisition of extremely large quantities of paired data, ensuring accurate alignment between image and text conditions is still an open research question in the community, as attested by the large amount of recent work in the domain~\cite{li2023blip2,zhai2023sigmoid}.
In practice, large web-scraped datasets, such as LAION-5B~\cite{schuhmann2022laion} or CC12M~\cite{Changpinyo_2021_CVPR}, contain abundant noisy pairs due to their collecting process.
To clean the pairs, hence ensuring alignment of higher quality, the prevailing strategy filters out samples that fail to meet an arbitrarily chosen criterion, often done through techniques like thresholding the CLIP-score~\cite{clip, pernias2023wuerstchen,podell2023sdxl}. 
This approach, however, has two main drawbacks: first, it is challenging to adjust the criterion accurately and more importantly, it discards many high-quality samples that could potentially enhance generation quality irrespective of the condition. For instance, out of the 50B initially-collected text-image pairs, only 10\% were left in LAION-5B~\cite{schuhmann2022laion}, thus discarding 90\% of the samples, i.e. 45B images. 

Instead of discarding the vast majority of training samples, in this work, we leverage them to learn simultaneously conditional and unconditional distributions. 
Specifically, we introduce a novel approach that estimates what we call the \emph{coherence score}, which evaluates how well the condition corresponds to its associated image. We incorporate this \emph{coherence score} into the training process by embedding it into a latent vector, which is subsequently merged with the condition. This additional information enables the diffusion model to determine the extent to which the condition should influence the generation of a target image. During inference, our method has the flexibility to take as input the coherence score, thereby allowing users to vary the impact of the condition on the generation process, as illustrated in Figure \ref{fig:teaser}. In addition, to further improve the generated image quality, we refine the Classifier-Free-Guidance method (CFG) introduced in~\cite{ho2022classifier} by leveraging the gap between high and low \emph{coherence scores}.

We evaluate our approach across three distinct tasks that involve various types of conditioning: text for text-to-image generation, labels for class-conditioned image generation, and semantic maps for paint-by-word image generation. In text conditioning, we use the CLIP score~\cite{clip} to estimate the coherence between the image and its accompanying caption. For class-conditional generation, we employ an off-the-shelf confidence estimator to gauge the coherence between the image and its label. Concerning semantic maps, we derive pixel-level coherence scores either by automatically generating them based on class boundaries or by using an off-the-shelf confidence estimator. Our evaluations span multiple datasets such as COCO~\cite{lin2014microsoft} for zero-shot text-to-image generation, ImageNet~\cite{5206848} for class-conditioned generation, and ADE-20K~\cite{zhou2017scene} for semantic maps. Our results show that including the coherence score in the training process allows training diffusion models with better image quality and that are more coherent with what they are prompted.

In summary, our contributions can be outlined as follows:
\begin{itemize}[left=0.4em]
\item[\checkmark] \textbf{Innovative Training Approach}: We present coherence-aware diffusion (CAD), a novel method for training conditional diffusion models in the presence of annotation imperfections. By incorporating a coherence score between the target image and its associated condition, our model can adapt and fine-tune the influence of the condition on the generation process.
\item[\checkmark] \textbf{Flexible Inference Scheme}: We introduce a versatile inference scheme, in which manual tuning of the coherence score during the generation enables the modulation of the condition's impact on image generation. Additionally, we refine the classifier-free guidance method under this new inference scheme, resulting in enhanced image quality.
\item[\checkmark] \textbf{Wide Applicability}: Demonstrating the versatility of CAD, we evaluate it across three diverse tasks involving different types of conditions (text, class labels, and semantic maps). CAD produces visually pleasing results across all tasks, emphasizing its generic applicability.
\end{itemize}

\section{Related Work}
\noindent \textbf{Conditional generation.}
Previous attempts to condition generative models were focused on GANs~\cite{GoodFellow2014}. Class-conditional~\cite{mirza2014conditional} was the first way to introduce conditioning in generative models.
This is a simple global conditioning. However, it lacks control over the output. 
To increase the control power of conditioning, more local conditionings were proposed such as drawings, maps, or segmentation maps~\cite{isola2017image}. Segmentation maps conditioning~\cite{park2019semantic, Zhu_2020_CVPR, dufour2022scam} propose the most control to the user. Indeed, the user can not only specify the shape of the objects but also
per-object class information. Semantic masks are however tedious to draw, which impacts usability. Text-conditioned models~\cite{zhang2017stackgan} offer a compromise. They can provide both global and local conditioning and are easy to work with. Recently, diffusion models %
have made great advances in this domain.

\noindent \textbf{Diffusion models.}
Diffusion models~\cite{sohl2015deep, song2019generative, song2020improved, ho2020denoising} have recently attracted the attention of research in image generation. Compared to GANs, they have better coverage over the data distribution,  are easier to train and outperform them in terms of image quality~\cite{dhariwal2021diffusion}. Architecture-wise, diffusion models rely mostly on modified versions of a U-Net~\cite{ho2020denoising, song2019generative, dhariwal2021diffusion}. Recent works have however shown that other architectures are possible~\cite{peebles2022scalable, hoogeboom2023simple}. In particular, RIN~\cite{jabri2022scalable} proposes a much simpler architecture than the U-Net achieving more efficient training.
They recently have been a lot of works~\cite{ramesh2022hierarchical, saharia2022photorealistic, balaji2022ediffi} scaling up these models on huge text-to-image datasets~\cite{schuhmann2022laion}. Stable Diffusion~\cite{rombach2022high}, Stable Diffusion XL~\cite{podell2023sdxl}, Paella~\cite{rampas2022fast} or Wuerstchen~\cite{pernias2023wuerstchen} have provided open-source weights for their networks, which has allowed an explosion in image generation. ControlNet~\cite{zhang2023adding} has shown that fine-tuning these models allows for very fine-grained control over the output with lots of different conditioning modalities. Recently, consistency models~\cite{song2023consistency, luo2023latent} have shown that by training more with a different loss, inference can be done in small amounts of steps (2-4 steps). All these text-to-image networks have been tuned on very noisy web-scrapped data. We argue in this paper that this noise causes limitations in the training. Concurrent works~\cite{BetkerImprovingIG, chen2023pixart, segalis2023picture,gu2023matryoshka} propose to tackle this task through re-captioning, but this requires lots of resources to train a good captioner that outputs detailed captions without hallucinating details. As shown by~\cite{BetkerImprovingIG}, it also requires bridging the gap between train and test time prompts. Instead, our approach is much simpler in current training setups.

\noindent \textbf{Learning with noisy conditioning}  has been widely explored when considering classification. For binary classification, \cite{natarajan2013learning} study machine learning robustness when confronted with noisy labels, while~\cite{ishida2018binary} train a DNN with exclusively positive labels accompanied by confidence scores. To bring a more practical perspective, \cite{berthon2021confidence} introduced instance-dependent noise scored by confidence, where this score aligns with the probability of the assigned label's accuracy. The negative impact of noisy labels has been mitigated with changes in architecture~\cite{Goldberger2016TrainingDN, cheng2020weakly}, in the loss~\cite{ren2018learning} or filtering the noisy samples~\cite{han2018co}. More recently, \cite{kang2022noise} propose a similar approach to ours by conditioning an image captioner model by the CLIP-score to mitigate the impact of text misalignment. Instead, we focus on image synthesis, where we condition the diffusion model with a coherence score.

\section{Coherence-Aware Diffusion (CAD)}
\label{sub:method}

In this work, we want to improve the training of the diffusion model in the presence of misaligned conditioning. We make the assumption that for each training sample, a coherence score measures how coherent the conditioning is with respect to the data. We propose to condition the diffusion model on this coherence score in addition to the original condition. By doing so, the model learns to discard the low-coherence conditions and focus on the high-coherence ones. Consequently, our model can behave as either a conditional or an unconditional model. Low-coherence samples, lead to unconditional sampling, while high-coherence samples lead to conditional samples. Building on this, we redesign classifier-free guidance to rely on coherence conditioning instead of dropping out the conditioning randomly.
\begin{figure}
\begin{center}
    \begin{minipage}{\linewidth}
        \begin{minipage}[c]{0.24\linewidth}
            \centering
            Baseline
        \end{minipage}
        \begin{minipage}[c]{0.24\linewidth}
            \centering
            Filtered
        \end{minipage}
        \begin{minipage}[c]{0.24\linewidth}
            \centering
            Weighted
        \end{minipage}
        \begin{minipage}[c]{0.24\linewidth}
            \centering
            CAD (Ours)
        \end{minipage}
        \hfill
        \begin{minipage}[c]{0.24\linewidth}
            \siximagesquare{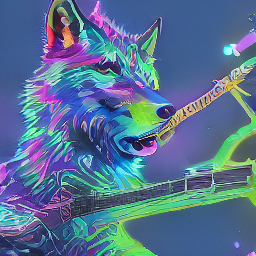}{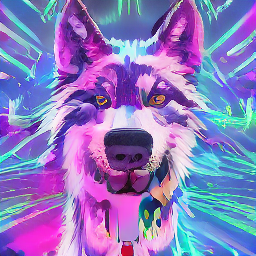}{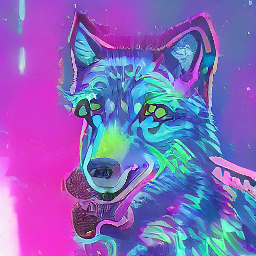}{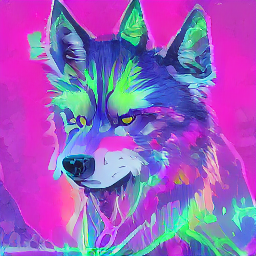}{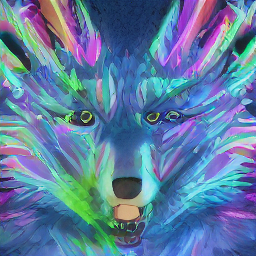}{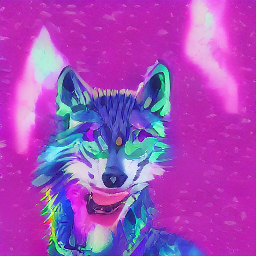}
        \end{minipage}
        \hfill
        \begin{minipage}[c]{0.24\linewidth}
            \siximagesquare{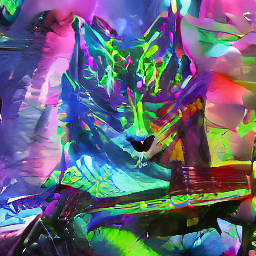}{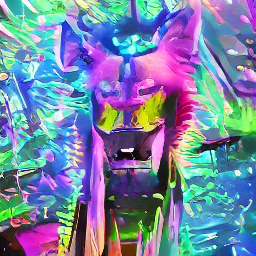}{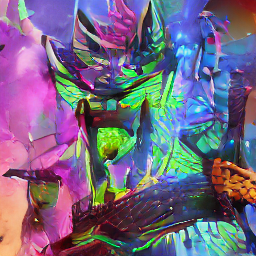}{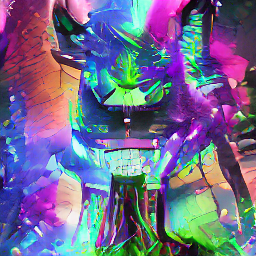}{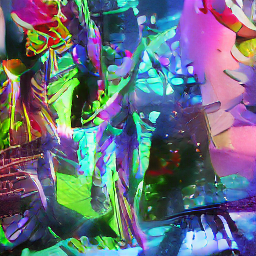}{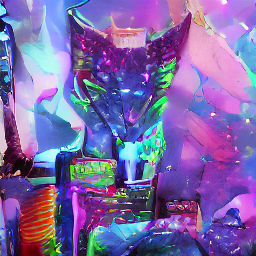}
        \end{minipage}
        \hfill
        \begin{minipage}[c]{0.24\linewidth}
            \siximagesquare{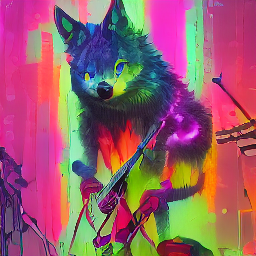}{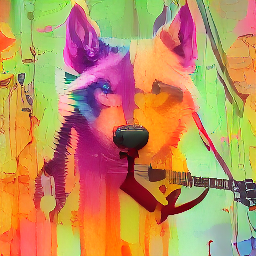}{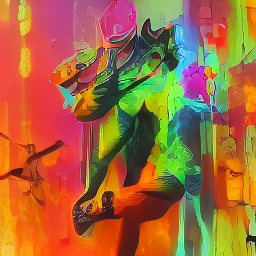}{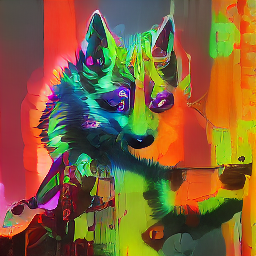}{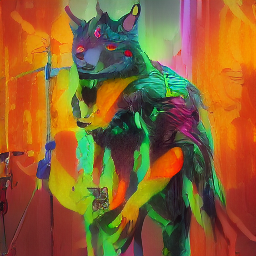}{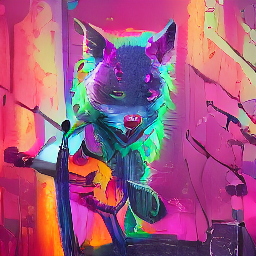}
        \end{minipage}
        \hfill
        \begin{minipage}[c]{0.24\linewidth}
            \siximagesquare{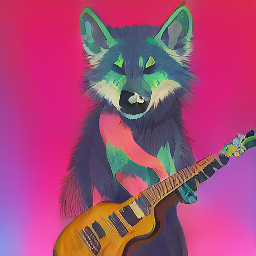}{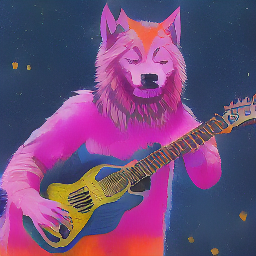}{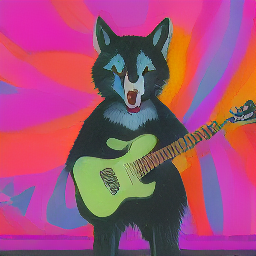}{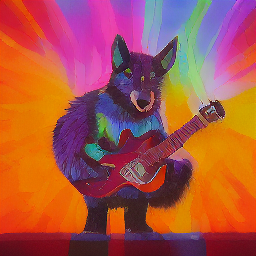}{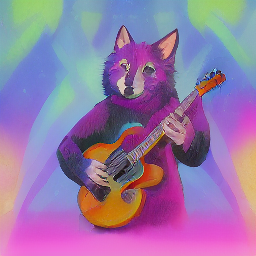}{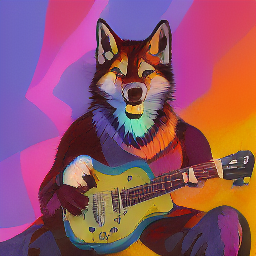}
            \vspace{1mm}
        \end{minipage}
        
    \end{minipage}
    \vspace{1mm}
    \small{\textit{A vibrant, multicolored furry \textcolor{vyellow}{\textbf{wolf}} with neon highlights playing an electric \textcolor{vyellow}{\textbf{guitar}} on stage; trending on artstation}}
    
    \vspace{0.2em}
    \begin{minipage}{\linewidth}
        \begin{minipage}[c]{0.24\linewidth}
            \siximagesquare{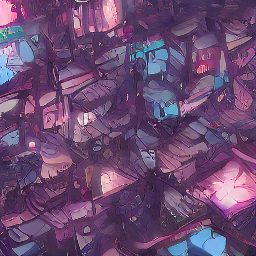}{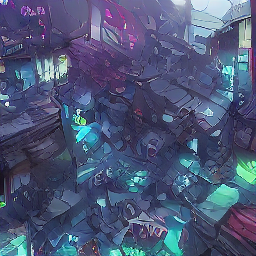}{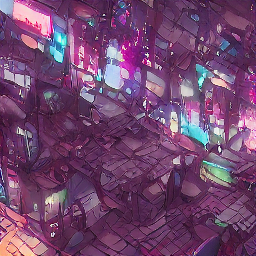}{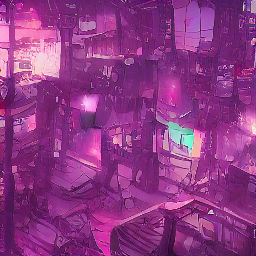}{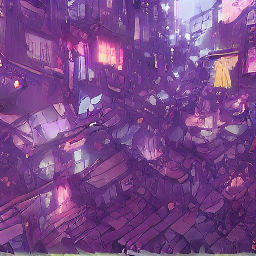}{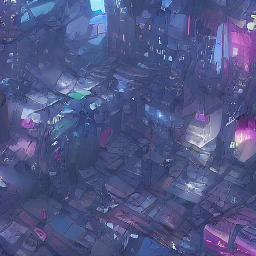}
        \end{minipage}
        \hfill
        \begin{minipage}[c]{0.24\linewidth}
            \siximagesquare{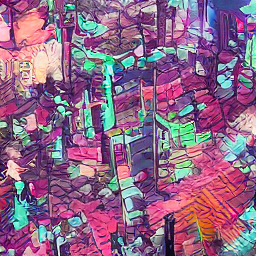}{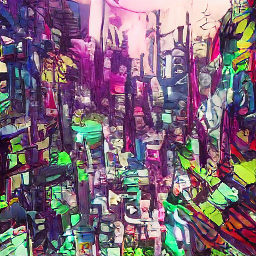}{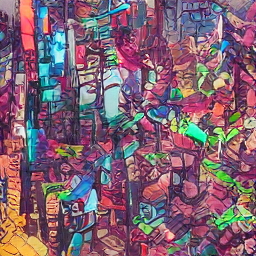}{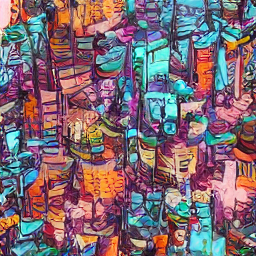}{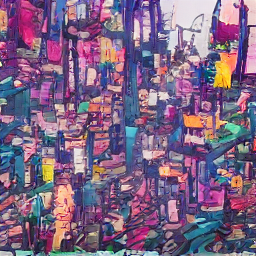}{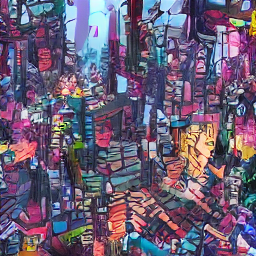}
        \end{minipage}
        \hfill
        \begin{minipage}[c]{0.24\linewidth}
            \siximagesquare{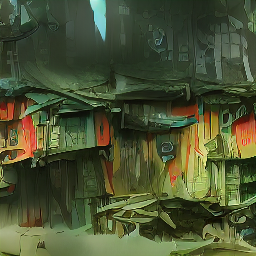}{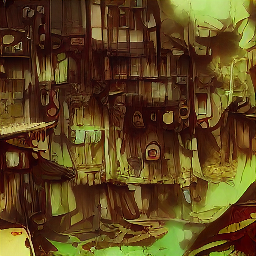}{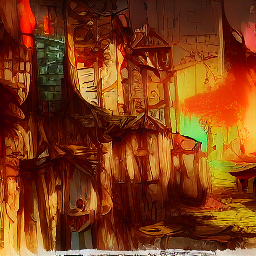}{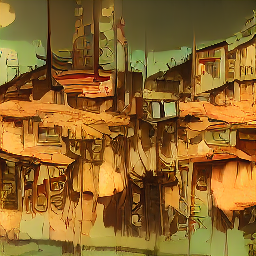}{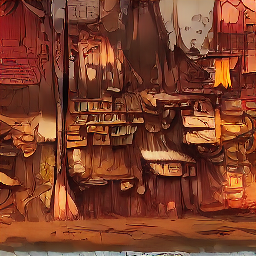}{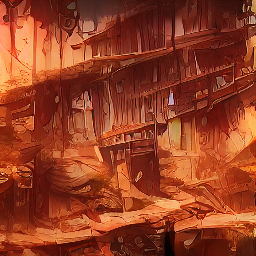}
        \end{minipage}
        \hfill
        \begin{minipage}[c]{0.24\linewidth}
            \siximagesquare{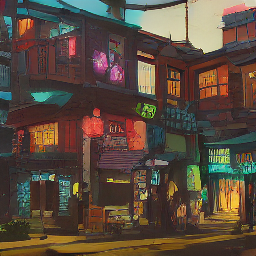}{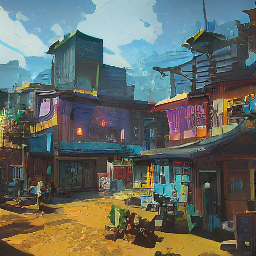}{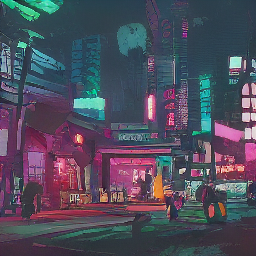}{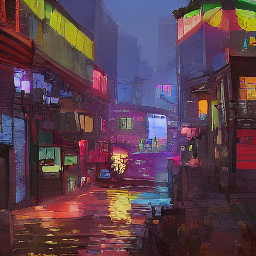}{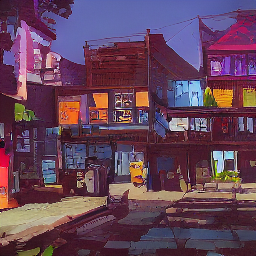}{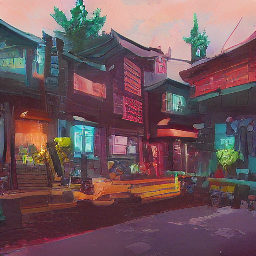}
            \vspace{1mm}
        \end{minipage}
    \end{minipage}
    \vspace{1mm}
    \small{\textit{A \textcolor{vyellow}{\textbf{shanty}} version of \textcolor{vyellow}{\textbf{Tokyo}}, new rustic style, bold colors with all colors palette, video game, genshin, tribe, \textcolor{vyellow}{\textbf{fantasy}}, overwatch}}
    \captionof{figure}{Images generated with a TextRIN trained with different handling of the misalignment between the image and its associated text at training. Compared to doing nothing (baseline), removing misaligned samples (filtering) or weighting the loss (weighted), our Coherence-Aware Diffusion training (CAD) generates more visually pleasing images while better adhering to the prompt.}
    \label{fig:qualitative}

\end{center}
\vspace{-2.5em}
\end{figure}

\subsection{Conditional Diffusion Models }

We first provide an overview of conditional diffusion models. These models learn to denoise a target at various noise levels. By denoising sufficiently strong noises, we eventually denoise pure noise, which can then be used to generate images. Each diffusion process is associated with a network $\epsilon_{\theta}$, which performs the denoising task. To train such a network, we have $X$ an image and $y$ its associated conditioning coming from $p_\text{data}$ the data distribution. We use a noise scheduler $\gamma(t)$, which defines $X_t$ which is the input image corrupted with Gaussian noise at the $t$-th step of diffusion. such as $X_t = \sqrt{\gamma(t)}X + \sqrt{1-\gamma(t)}\epsilon$, where $\epsilon \sim \mathcal{N}(0,1)$ the noise we want to predict and $t \in [0,1]$ the diffusion timestep. During training, conditioning is provided to $\epsilon_{\theta}$.
The objective of the diffusion model is to minimize the following loss:
\begin{equation}
    L_\text{simple} = \mathbb{E}_{(X,y) \sim p_\text{data}, t \sim \mathcal{U}[0,1]}[\left\lVert \epsilon - \epsilon_{\theta}(X_t, y, t)\right\rVert] \quad ,
\end{equation}
where $\left\lVert \cdot \right\rVert$ denotes the $L_2$ norm.

One observation is that the conditioning is implicitly learned by the diffusion model, as the diffusion loss is only enforced on the image and not on the conditioning itself.
This motivates our hypothesis that removing data with low label coherence can harm the training of the diffusion model. Even if the conditioning is not well aligned, the image still belongs to the distribution that we aim to learn. By discarding such data, we weaken the distribution estimator.

\begin{figure}[b]
    \centering
    \includegraphics[width=\columnwidth]{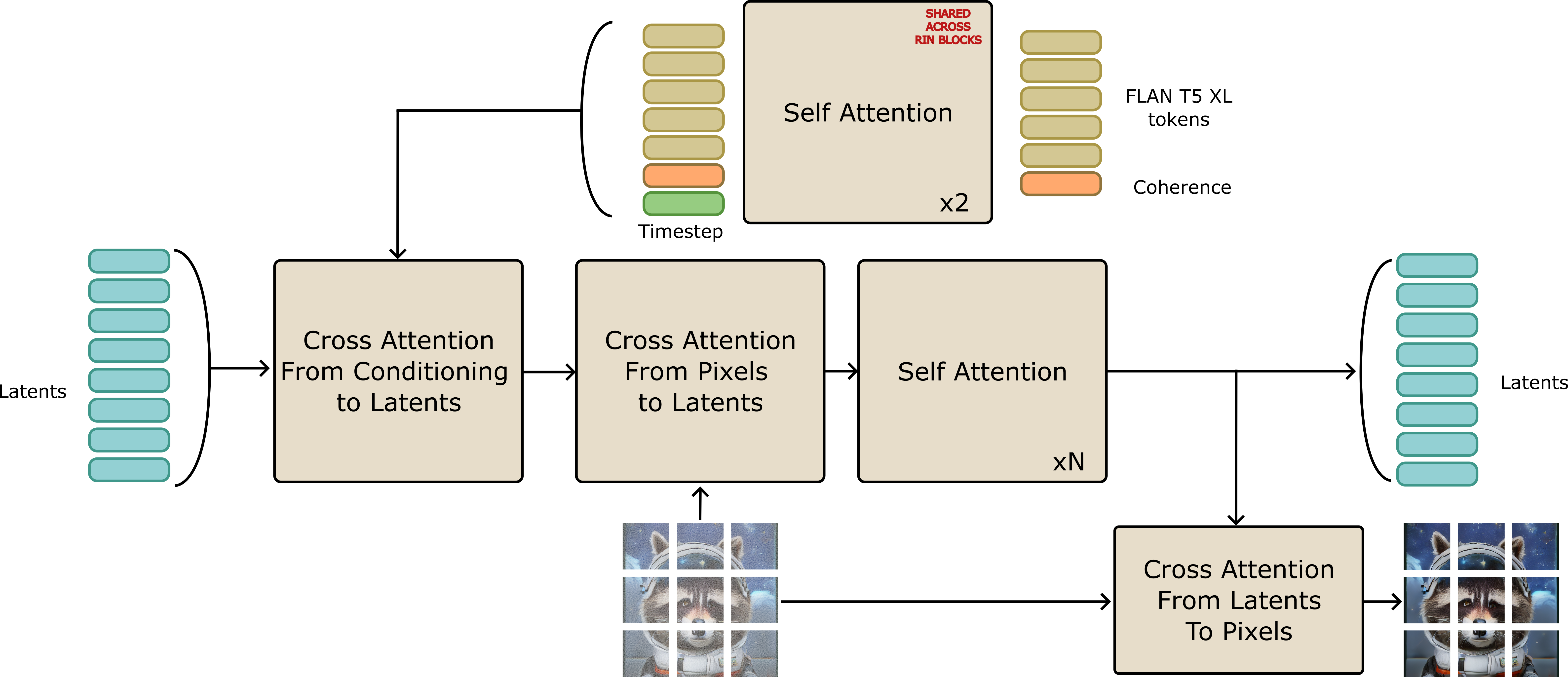}
    \caption{\textbf{Text RIN Block.} Architecture of the proposed Text RIN Block used in CAD. We include a cross attention from the text to the latent branch of the RIN block.}
    \label{fig:text_rin_block_arch}
\end{figure}
\begin{figure*}[ht]
    \centering
    \begin{minipage}{0.49\textwidth}
    \renewcommand{\arraystretch}{1.5}
    \resizebox{\textwidth}{!}{
        \begin{tabular}{cc | cccccc|}
         & & \multicolumn{6}{| c |}{\cellcolor{lavenderblue}{\textbf{COCO-10K}}}\\
         \hline
         \textbf{Method} & \textbf{$\omega$}  & \textbf{$\text{FID}_{\text{CLIP}}$ $\downarrow$} & \textbf{$\text{CLIPScore}$ $\uparrow$} & \textbf{$\text{P}_{\text{CLIP}}$ $\uparrow$} & \textbf{$\text{R}_{\text{CLIP}}$ $\uparrow$} & \textbf{$\text{D}_{\text{CLIP}}$ $\uparrow$} & \textbf{$\text{C}_{\text{CLIP}}$ $\uparrow$}\\
        \hline
         Baseline & 10 & 91.9 & 25.96 & \underline{0.281} & 0.047 & \underline{0.181} & 0.222\\
         Weighted & 5 & 98.3 & 25.15 & 0.192 & 0.046 & 0.111 & 0.155 \\
         Filtered & 10 & \underline{85.8} & \textbf{26.52} & \underline{0.281} & \underline{0.061} & 0.175 & \underline{0.233} \\
         CAD-S (Ours) & 15 & \textbf{69.4} & \underline{26.16} & \textbf{0.373} & \textbf{0.078} & \textbf{0.265} & \textbf{0.315} \\ \hline
        CAD-B (Ours) & 15 & 55.3 & 27.04 & 0.462 & 0.172 & 0.349 & 0.428\\  
        CAD-B 512px (Ours) & 15 & 54.3 & 25.52 & 0.593 & 0.148 & 0.545 & 0.472 \\
        \end{tabular}
        }
    \captionof*{table}{(a)}
    \label{tab:coco_quant}
    \end{minipage}
    \hfill
    \begin{minipage}{0.23\textwidth}
  \begin{center}
    \resizebox{\textwidth}{!}{
    \includegraphics[width=\linewidth]{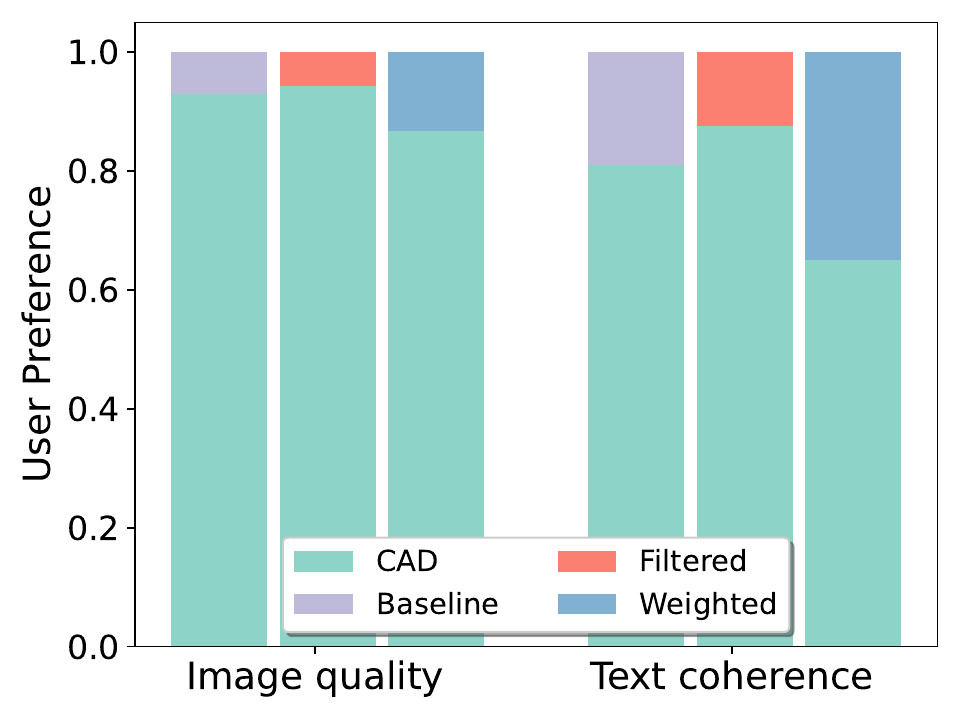}
    }
    \captionof*{figure}{(b)}
    \label{fig:user_study}
    \end{center}
    \end{minipage}
    \hfill
    \begin{minipage}{0.22\textwidth} 
        \begin{center}
        \includegraphics[width=\linewidth]{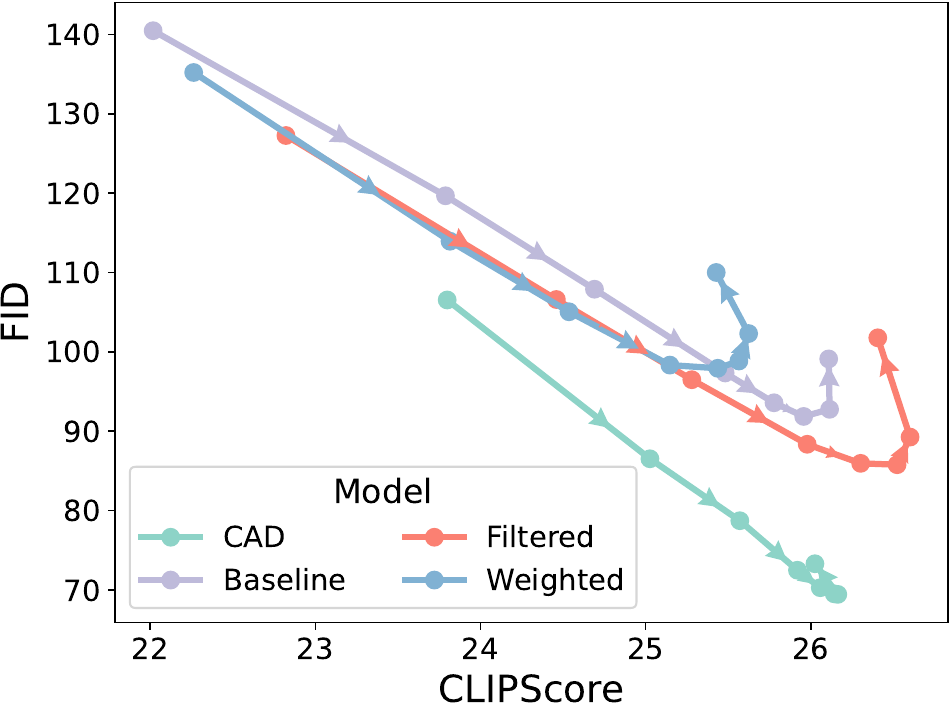}
        \captionof*{figure}{(c)}
        \label{fig:fid_vs_clip}
        \end{center}
    \end{minipage}
    \caption{\textbf{Text-to-image generation results.} (a) Quantitative results for text-to-image generation. We show that CAD achieves significantly lower FID, precision, recall, density and coverage while keeping similar CLIP score. (b) User study results. Users had to indicate the highest quality image and the most adhering to the prompt among pairs of images corresponding to our CAD method and one of baseline, filtered or weighted method. (c) FID versus CLIP on the text-to-image task for varying degrees of guidance $\omega$. We show that CAD achieves a significantly better trade-off with a much lower FID for the same CLIP score.}
    \label{fig:text-to-image}
\end{figure*}

\subsection{Integrating label information into the diffusion model}\label{CAD_method}
We assume that for every datapoint $(X, y)$ we have an associated $c$, the coherence score of $y$ where $c\in[0,1]$. Our goal is to incorporate label coherence into the diffusion model to discard only the conditioning that contains low levels of coherence while continuing to train on the image. A value of $c{=}1$ indicates that $y$ is the best possible annotation for $X$, while $c{=}0$ suggests that $y$ is a poor annotation for $X$.

To achieve this, we modify the conditioning of the diffusion model $\epsilon_{\theta}$ to include both $y$ and $c$, using the following loss:
\begin{equation}
L_\text{simple} = \mathbb{E}_{(X,y,c) \sim p_\text{data}, t \sim \mathcal{U}[0,1]}[\left\lVert \epsilon - \epsilon_{\theta}(X_t, y, c, t)\right\rVert] \quad .
\end{equation}
We refer to this kind of models as coherence coherence-aware diffusion (CAD) models. By informing the diffusion model of the coherence score associated with samples, we avoid filtering out low-confident samples and let the model learn by itself what information to take into account. Avoiding the filtering allows us to still learn $X$ even in the presence of noisy labels.

\subsection{Test-time prompting}
After training a model with different levels of coherence, we can thus prompt it with varying degrees of coherence. 
When we prompt with minimal coherence, we obtain an unconditional model. 
On the other hand, when we prompt with maximal coherence, we get a model that is very confident about the provided label. 
However, like any other conditional diffusion model relying on attention, there is no guarantee that the label is actually used.

To strengthen the use of the label, we propose a modification to the Classifier Free Guidance (CFG) method~\cite{ho2022classifier} that leverages the coherence. CFG uses both a conditional and unconditional model to improve the quality of generated samples. To learn such models, a conditional diffusion model is used and the conditioning is dropped out for a portion of the training samples. The original CFG formulation is as follows:
\begin{equation}
\hat{\epsilon}_{\theta}(x_t, y) = \epsilon_{\theta}(x_t, y) +\omega(\epsilon_{\theta}(x_t, y) - \epsilon_{\theta}(x_t, \emptyset)) \quad ,
\end{equation}
with $\omega$ the guidance rate. Instead, we propose a coherence-aware version of CFG (CA-CFG):
\begin{equation}
\hat{\epsilon}_{\theta}(x_t, y) = \epsilon_{\theta}(x_t, y, 1) +\omega(\epsilon_{\theta}(x_t, y, 1) - \epsilon_{\theta}(x_t, y, 0)) \quad .
\end{equation}
This modification removes the need to dropout the conditioning. Instead, we directly use the noise in the conditioning to drive the guidance.

\section{Experiments}
In this section, we will analyze 3 tasks: text, class, and semantically conditioned image generation. We describe the experimental setup, and analyze quantitative and qualitative results to better understand the inner workings of Coherence-aware diffusion.

\subsection{Experimental setup and Metrics}
\paragraph{Experimental setup.}
\emph{For text-conditional image generation}, we use a modified version of RIN~\cite{jabri2022scalable} (See Figure~\ref{fig:text_rin_block_arch}). To map the text to an embedding space, we use a frozen FLAN-T5 XL~\cite{flant5}. We then map the embedding with 2 self-attention transformer layers initialized with LayerScale~\cite{touvron2021going} and 16 registers. The text tokens are mapped with the coherence score using these layers, which is the same for every Text RIN block. Unlike the class conditional RIN block, the latent branch contains only the latents without concatenated tokens. Instead, the mapped text tokens, coherence, and timestep embeddings provide information to the latent branch with a cross-attention layer at the beginning of each Text RIN block. The rest of the architecture follows the standard RIN block design. The CAD-S model has 188M parameters and we train it for 240K steps. We train these models on a mix of datasets composed of CC12M~\cite{changpinyo2021cc12m} and LAION Aesthetics 6+~\cite{schuhmann2022laion}. We also train a 332M parameters version for 500k steps that we name CAD-B. We also finetune it for 200K steps at 512px resolution. To estimate the coherence score, we use MetaCLIP H/14~\cite{xu2023demystifying} that we then bin into 8 equally distributed discrete bins. We then use the normalized index between 0 and 1 as the coherence score. We compare our method to 3 baselines: "Baseline" is a model where we just train without coherence, "Filtered" corresponds to a model where we discard the 3 less coherent bins, and "Weighted" corresponds to a model where we weight the loss of the model by the normalized coherence of the sample. When using Coherent Aware prompting, we sample a random set of characters that we use with coherence-score of zero as the negative prompt.

\noindent\emph{For the class-conditional image generation experiments}, we 
rely on RIN~\cite{jabri2022scalable} and use the same hyperparameters as the authors. We experiment on conditional image generation for CIFAR-10~\cite{Krizhevsky2009LearningML} and Imagenet-64~\cite{5206848} for which we artificially noise the label. We extract the coherence score from pre-trained classifiers in the following way: We re-sample with some temperature $\beta$ a new label from the label distribution predicted by the classifier. We then consider the entropy of the distribution as the coherence score. After, we use a sinusoidal positional embedding~\cite{vaswani2017attention} that we map with an MLP. We add this coherence token in the latent branch of RIN, similar to the class token.
\\
\noindent\emph{For semantic segmentation conditioned experiments}, we use ControlNet~\cite{zhang2023adding} to condition a pre-trained text-to-image Stable-Diffusion~\cite{rombach2022high} model with both semantic and coherence maps concatenated. The training and evaluation of our method are performed on the ADE20K dataset~\cite{zhou2017scene}, a large-scale semantic segmentation dataset containing over 20,000 images with fine-detailed labels, covering diverse scenes and classes. Since captions are not available for this dataset, we use BLIP2~\cite{li2023blip} to generate captions for each image in the dataset similarly to~\cite{zhang2023adding}. We use a pre-trained Maskformer~\cite{cheng2021maskformer} on the COCO Dataset~\cite{lin2014microsoft}, to extract the segmentation map and its associated confidence (MCP~\cite{hendrycks2017a}) for each image in the ADE20k dataset\footnote{It is worth noting that using a Maskformer trained on the same dataset would result in high confidence map everywhere due to its high performance on the training set~\cite{guo2017calibration}}. We use confidence as our coherence score. More details about the experimental setup are available in the supplementary.

\begin{figure*}[t]
    \centering
    \begin{minipage}[t]{0.24\textwidth}
        \centering
        \subcaption*{$\omega = 0.0$}
        \includegraphics[width=\textwidth]{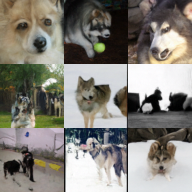}
        
    \end{minipage}
    \hfill
    \begin{minipage}[t]{0.24\textwidth}
        \centering
        \subcaption*{$\omega = 1.0$}
        \includegraphics[width=\textwidth]{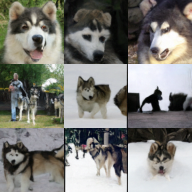}
    \end{minipage}
    \hfill
    \begin{minipage}[t]{0.24\textwidth}
        \centering
        \subcaption*{$\omega = 5.0$}
        \includegraphics[width=\textwidth]{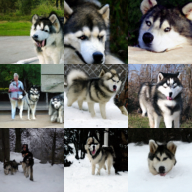}
    \end{minipage}
    \hfill
    \begin{minipage}[t]{0.24\textwidth}
        \centering
        \subcaption*{$\omega = 20.0$}
        \includegraphics[width=\textwidth]{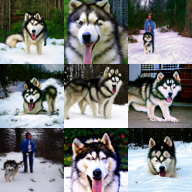}
    \end{minipage}
    \\
    \vspace{0.1cm}
    \begin{minipage}[t]{0.24\textwidth}
        \centering
        \includegraphics[width=\textwidth]{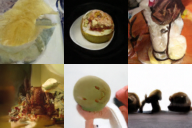}
    \end{minipage}
    \hfill
    \begin{minipage}[t]{0.24\textwidth}
        \centering
        \includegraphics[width=\textwidth]{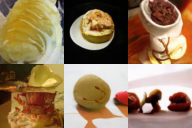}
    \end{minipage}
    \hfill
    \begin{minipage}[t]{0.24\textwidth}
        \centering
        \includegraphics[width=\textwidth]{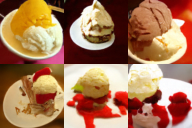}
    \end{minipage}
    \hfill
    \begin{minipage}[t]{0.24\textwidth}
        \centering
        \includegraphics[width=\textwidth]{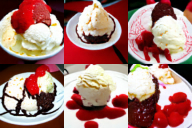}
    \end{minipage}
    \\
    \caption{Coherence-Aware Classifier-free Guidance for classes Malamute and Ice Cream with guidance rates $\omega \in \{0, 1, 5, 20\}$.}
    \label{fig:cacfg_images}
\end{figure*}
\paragraph{Metrics.} To evaluate image generation for all condition types, we use the Frechet Inception Distance~\cite{heusel2017gans} (FID) that evaluates image quality. We also use Precision~\cite{kynkaanniemi2019improved} (P), Recall~\cite{kynkaanniemi2019improved} (R), Density~\cite{naeem2020reliable} (D) and Coverage~\cite{naeem2020reliable}  (C) as manifold metrics, allowing us to evaluate how well the manifold of the generated images overlaps with the manifold of the real images.
For text-conditional, we also compute the CLIP Score~\cite{clip} and evaluate metrics on CLIP features on a 10K samples subset of COCO~\cite{lin2014microsoft} in a zero-shot setting.
For class-conditional, we compute the Inception Score~\cite{salimans2016improved} (IS). We also add the Accuracy (Acc) metric aiming at evaluating how well the image generator takes into account the conditioning and defined as $Acc(g) = \mathbb{E}_{c\in \text{Cat}(N)}[\mathbb{1}_{f(g(c)) = c}]$, where $g(.)$ is the generator we want to evaluate, $f(.)$ is a classifier, and $\text{Cat}(N)$ is the categorical distribution of $N$ labels. For CIFAR-10, we use a Vision Transformer~\cite{dosovitskiy2020image} trained on CIFAR-10, and for ImageNet, we use a DeiT~\cite{touvron2021training}.
For segmentation, instead of Accuracy, we compute the mean Intersection over Union (mIoU) instead of the Accuracy. 

\begin{figure*}

\begin{minipage}[b]{0.37\textwidth}
    \centering
     \includegraphics[width=\linewidth]{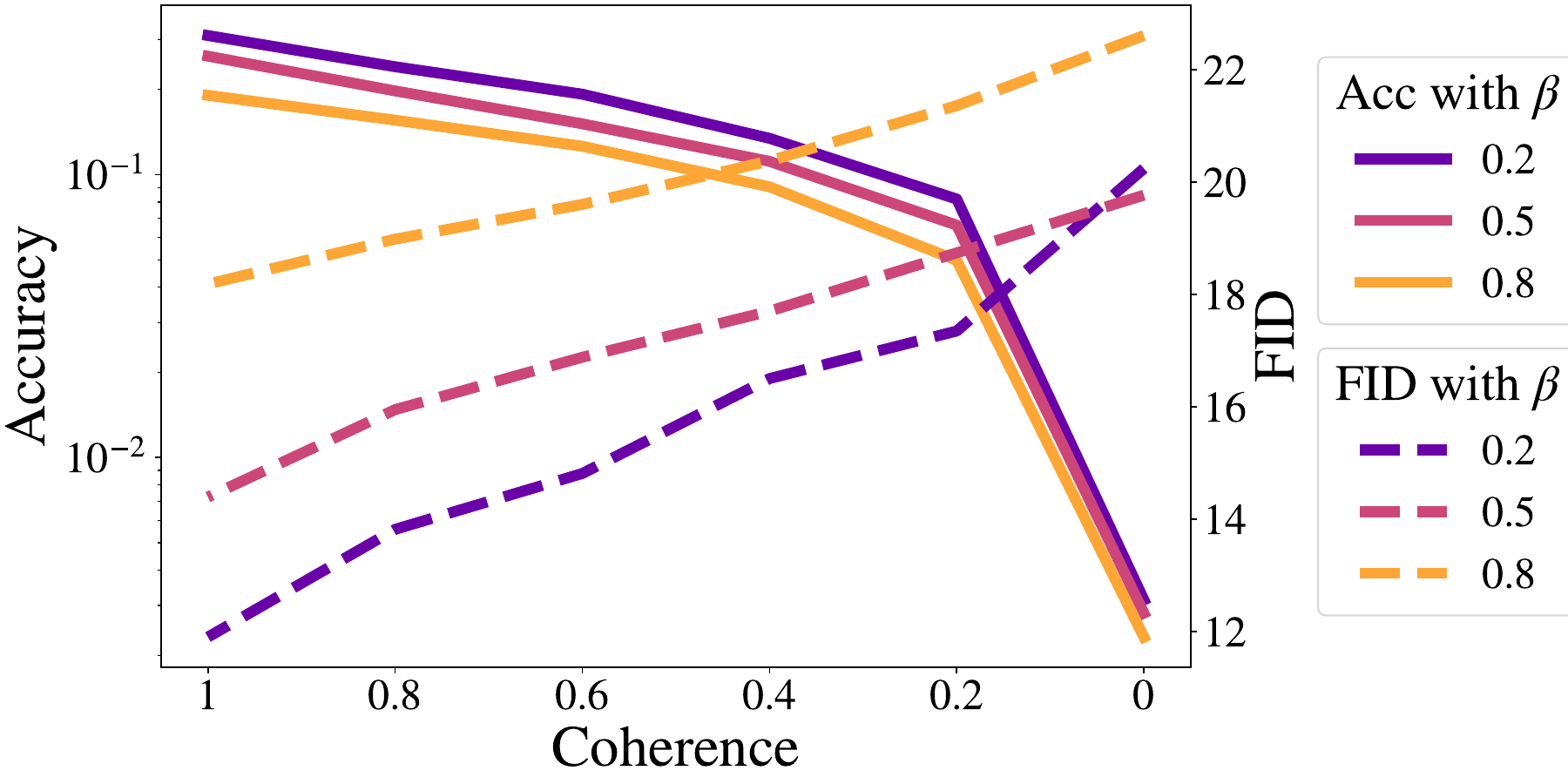}
    \captionof*{figure}{(a)}
\end{minipage}
\hfill
\begin{minipage}[b]{0.24\textwidth}
    \centering
     \includegraphics[width=\linewidth]{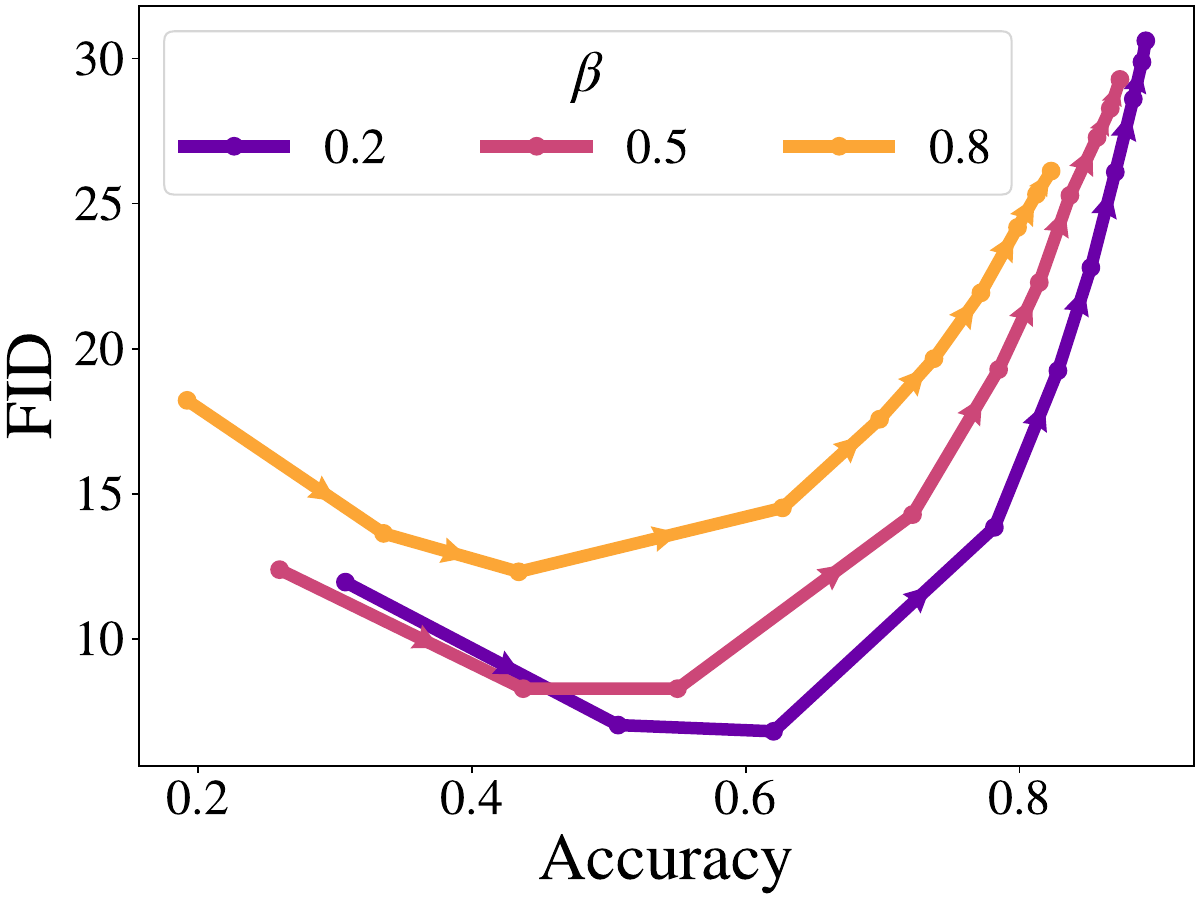}
     \captionof*{figure}{(b)}
\end{minipage}
\hfill
\begin{minipage}[b]{0.31\textwidth}
    \centering
     \includegraphics[width=\linewidth]{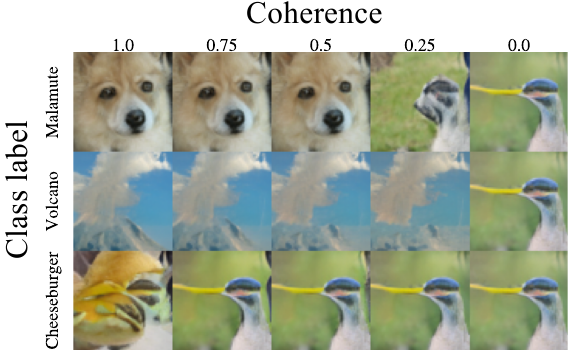}
    \captionof*{figure}{(c)}
\end{minipage}
\\
\begin{minipage}[b]{0.71\textwidth}
\centering
{
\renewcommand{\arraystretch}{1.4}
\resizebox{\textwidth}{!}{
\begin{tabular}{cc | ccccccc | ccccccc|}
 & & \multicolumn{7}{| c |}{\cellcolor{lavenderblue}{\textbf{ImageNet}}} & \multicolumn{7}{| c |}{\cellcolor{lavenderblue}{\textbf{CIFAR-10}}} \\
 \hline
 \textbf{$\beta$} & \textbf{Method} & \textbf{FID}$\downarrow$ & \textbf{IS}$\uparrow$ & \textbf{Acc}$\uparrow$ & \textbf{P}$\uparrow$ & \textbf{R}$\uparrow$ & \textbf{D}$\uparrow$ & \textbf{C}$\uparrow$ & \textbf{FID}$\downarrow$ & \textbf{IS}$\uparrow$ & \textbf{Acc}$\uparrow$ & \textbf{P}$\uparrow$ & \textbf{R}$\uparrow$ & \textbf{D}$\uparrow$ & \textbf{C}$\uparrow$ \\
\hline
 & Conditional & 7.56 & 34.26 & 0.475 & 0.595 & 0.610 & 0.768 & 0.706 & 8.97 & 10.13 & 0.954 & 0.630 & 0.573 & 0.817 & 0.758\\
\hline
\multirow{3}{*}{0.5} & Baseline & 14.38 & 20.46 & 0.168 & 0.539 & 0.579 & 0.595 & 0.505 & \underline{5.66} & \textbf{9.79} & 0.507 & 0.659 & \textbf{0.614} & 0.940 & \underline{0.809} \\
 & Filtered & \underline{10.20} & \textbf{26.59} & \textbf{0.338} & \textbf{0.573} & \underline{0.608} & \underline{0.707} & \textbf{0.645} & 9.48 & 9.53 & \underline{0.634} & \underline{0.679} & 0.547 & \underline{1.021} & 0.789\\
 & CAD & \textbf{9.11} & \underline{25.97} & \underline{0.327} & \underline{0.571} & \textbf{0.610} & \textbf{0.714} & \underline{0.633} & \textbf{4.75} & \underline{9.69} & \textbf{0.906} & \textbf{0.688} & \underline{0.588} & \textbf{1.059} & \textbf{0.821}\\
\end{tabular}
}
    }
    \captionof*{table}{(d)}
\end{minipage}
\hfill
\begin{minipage}[b]{0.26\textwidth}
    \centering
     \includegraphics[width=0.97\linewidth]{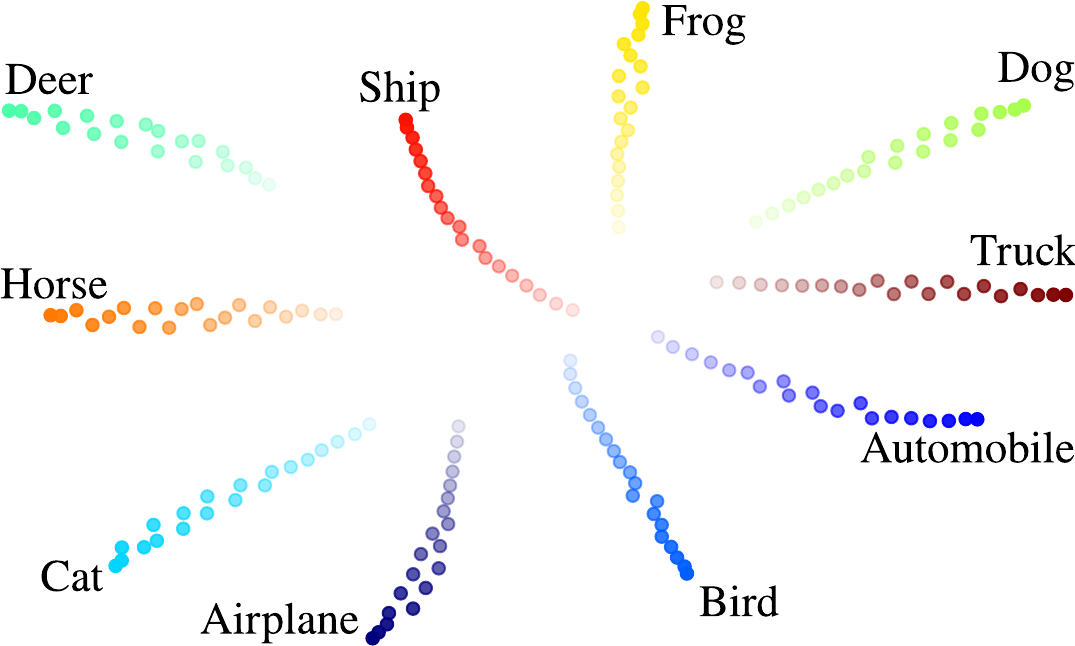}
    \captionof*{figure}{(e)}
\end{minipage}
\caption{\textbf{Impact of coherence on the model.} Top: (a) Impact on FID and Accuracy of prompting a model (CAD on ImageNet with $\beta=0.5$). (b) FID vs Accuracy (using CA-CFG). We vary the guidance rate from 0 to 25. (c) Impact of prompting with different coherence scores on image generation. Low coherence indicates convergence towards an unconditional model. Bottom: (d) Quantitative results for class-conditional image generation. Our coherence aware diffusion (CAD) is compared to a baseline model and a training set filtering strategy for different levels of label noise $\beta$. We show that CAD achieves higher fidelity and better accuracy. (e) TSNE of a mixed embedding of the class label and the coherence score on CIFAR-10. Each color denotes a class and the transparency shows the coherence level: the more transparent, the less coherence. }
\label{tab:class_coherence}
\end{figure*}

\subsection{Analysis}

\paragraph{Coherence conditioning.}
Here, we explore the behavior of our proposed coherence-aware diffusion model at test time. 
For the text conditional setting, we observe from Figure~\ref{fig:clip_vs_confidence} that the coherence and the quality of the generated image increase as the coherence increases. Indeed, FID decreases and the CLIPScore increases. 
In the class-conditional setup, we prompt a CAD model trained on resampled ImageNet and report results in Table~\ref{tab:class_coherence}(a). Similar to the previous setting, when the model has very high coherence, it achieves the best FID and accuracy. However, when the coherence score decreases, the accuracy decreases as well and drops to $\frac{1}{N}$ when the coherence goes to $0$. This validates our hypothesis that in the presence of low-coherence samples, our proposed model behaves like an unconditional model. Furthermore, even if the FID increases, it remains close to the FID of the conditional model, which implies that our CAD samples images are close to the training distribution.

Qualitatively, for text-to-image generation, we prompt the model with varying coherence scores from 0 to 1 and display results in Figure~\ref{fig:clip_vs_conf_viz}. We observe that when the coherence increases, the outputs are close to the prompt. For instance, in the bottom figure, the generated image displays an avocado armchair, where avocado and armchair are successfully mixed. Even a more complex prompt, like the raccoon at the top, follows closely the textual description. The raccoon does wear an astronaut suit and is looking through the window at a starry night. Similarly, as the coherence decreases, the images start to diverge from the original prompt. The avocado chair starts to first lose the "avocado" traits until there is only a chair and at the end an object that does not look like an avocado or a chair. At the top, we first lose the window, then the raccoon. We note that contrary to class conditional (as seen below), we do not converge to a totally random image. Instead, some features from the prompt are preserved, such as the racoon but the global structure of the image gets lost. This is highly linked to the CLIP network biases, which may pay less attention to less salient parts of an image such as the background, and are more sensitive to the main subject.

Similarly, for class-conditional, we prompt a CAD model trained on resampled ImageNet, with different coherences and classes. We sample with DDPM but we seed the sampling to have the same noise when sampling different classes. Table~\ref{tab:class_coherence} (c) illustrates the results, where we observe that when prompted with high coherence, CAD samples have the desired class. However, as the coherence decreases, the samples get converted into samples from random classes. Furthermore, samples that use the same sampling noise converge towards the same image in the low coherence regime. This shows that when the label coherence is low, CAD discards the conditioning and instead samples unconditionally. 

To better understand the underlying mechanism, we design the following experiment. We modify the proposed model so that instead of adding both a class and a coherence token to the RIN network, we merge them into a single token with an MLP. Figure~\ref{tab:class_coherence} (e) displays the t-SNE plots of the output of the MLP for every class in CIFAR-10 for which we compute different coherence scores. In the plot, high coherence translates to low transparency. We observe that as the coherence decreases, the embeddings of all classes tend to converge into the same embedding (center). This corroborates our hypothesis that the model uses the coherence to learn by itself how much to rely on the label.

\noindent \textbf{Coherence-aware classifier-free guidance.} 
Here, we examine the impact of the coherence-aware classifier guidance. For this, we first compute different guidance rates ranging from 0 to 30 with 250 steps of DDIM, and then we plot the FID vs the CLIPScore the classifier accuracy for different rates. Figure~\ref{tab:class_coherence} (b) illustrates this. We observe a trade-off between classifier accuracy and FID. Specifically, the more we increase the guidance, the more the accuracy increases, but at the cost of higher FID. 

Qualitatively, this behaviour is also present in Figure~\ref{fig:cacfg_images}: at a lower guidance rate, the images are more diverse but at the cost of lower accuracy with respect to the class. %
This pattern is best shown in the Malamute example when $\omega=20$ (first part in Figure~\ref{fig:cacfg_images}), where all malamutes have a similar pose with their tongues hanging and similarly, and in the third part where all ice-creams look similar, i.e., one white ice-cream scoop with red fillings.

Interestingly, we also observe that some guidance leads to optimal results. In Figure~\ref{tab:class_coherence}, when $\omega=1$, the FID is at its lowest point, and the accuracy is higher than the default model that has $\omega=0$. This is also shown in Figure~\ref{fig:cacfg_images}, where, when $\omega=1$, samples best combine diversity and fidelity.

\subsection{Results}

In this section, we report image generation results conditioned on text, class, and semantic maps.

\paragraph{Qualitative results for text-conditioned image generation.} 
In Figure~\ref{fig:qualitative}, we observe that coherence-aware diffusion for textual conditioning allows for better prompt adherence and better-looking images. In the first row, we observe that CAD is the only method that captures the details of the prompt, such as having the wolf play the guitar. Indeed, the baseline and filtered models output a wolf, but most generations display only a head. The weighted model performs slightly better but it lacks quality. Furthermore, our model displays higher diversity in the output styles. For instance, this is visible in the bottom row where our model displays a variety of street images whereas the other methods tend to have a collapsed output. For this prompt, our model also displays better image quality compared to other methods.

\begin{figure*}[h!]
    \begin{minipage}{0.7\textwidth}
        \centering
     \includegraphics[width=0.97\linewidth]{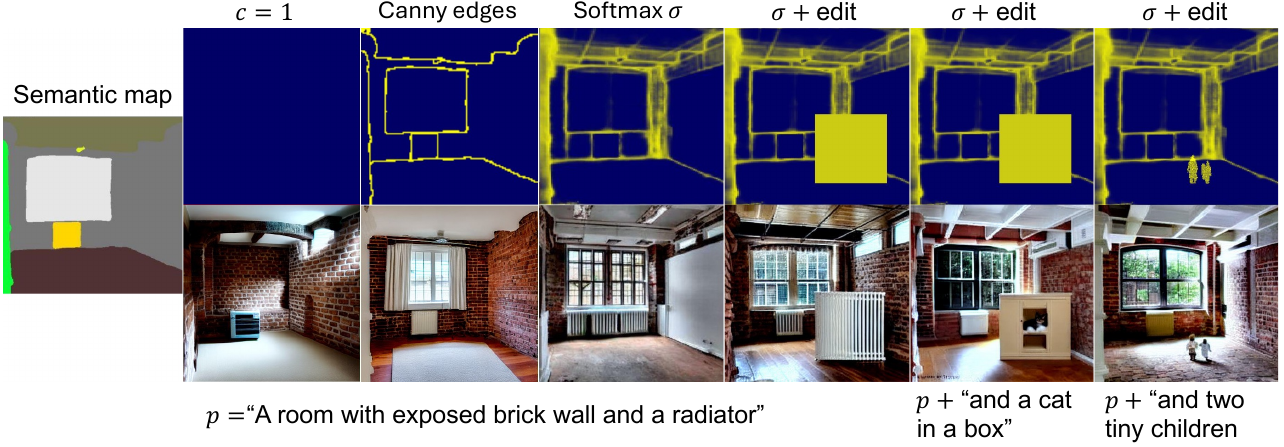}
    \caption*{\textbf{(a) Examples of image generation conditioned on a semantic map.}}
    \end{minipage}   
    \hfill
    \begin{minipage}{0.3\textwidth}
        \centering
        \resizebox{\textwidth}{!}{
    \begin{tabular}{c | cc | cc}
    \toprule 
    & \multicolumn{2}{c|}{\cellcolor{lavenderblue}{\textbf{ADE20k}}} & \multicolumn{2}{c}{\cellcolor{lavenderblue}{\textbf{COCO}}} \\ [1.5pt]
    Metric & Baseline & CAD & Baseline & CAD\\ [1.5pt]
    \hline 
    
    \textbf{FID}  & 33.67 & \textbf{30.88} & 20.1 & \textbf{18.1} \\[1.5pt]
    \textbf{mIoU} & 22.6 & \textbf{23.7} & 35.1 & \textbf{35.3} \\[1.5pt]
    \textbf{P}    & 0.785 & \textbf{0.844} & 0.7876 & \textbf{0.8404} \\[1.5pt]
    \textbf{R}    & 0.757 & \textbf{0.824} & 0.6760 & \textbf{0.8060} \\[1.5pt]
    \textbf{D}    & 1.029 & \textbf{1.0755} & \textbf{1.0811} & 1.0687 \\[1.5pt]
    \textbf{C}    & 0.904 & \textbf{0.934} & 0.8956 & \textbf{0.9304} \\[1.5pt]
    \bottomrule
\end{tabular}
        }
        \captionsetup{skip=10pt} %
        \caption*{\textbf{(b) Quantitative results on ADE20k and MS COCO}. We report FID, mIoU, Precision (P), Recall (R), Diversity (D) and Coverage (C).}
        
    \end{minipage}
\caption{\textbf{(a) Image generation conditioned on a semantic map.} The images generated for a given prompt $p$ (shown below) are shown with respect to different pixel-level coherence scores $c$ (shown above). Coherence scores are obtained either synthetically using Canny edge detection on the semantic map, or from the maximum of the softmax probability $\sigma$ of a pre-trained model. They can then be edited either manually ($\sigma$ + edit) or by blitting other coherence maps (rightmost).
    \textbf{(b) Quantitative results on ADE20k and MS COCO.}}
    \label{fig:semantic_seg}
\end{figure*}

\paragraph{Quantitative results for text-conditioned image generation.} In Figure~\ref{fig:text-to-image} (c), we compare our method to a classical diffusion model (baseline), to a method where all samples with CLIPScore lower than 0.41 are filtered as is commonly done with LAION-5B (Filtered), and one where the diffusion loss is wheighted by the CLIPScore of the sample (Weighted). We vary the coherence and plot the FID with respect to the CLIPScore. We observe that our method achieves significantly better FID/CLIP tradeoff than the other methods. In Figure \ref{fig:text-to-image} (a), we report the metrics for the guidance parameter that gives the best FID. We observe that our method outperforms other methods on all metrics except for the CLIPScore. In particular, our method achieves a better FID by 15 points than the second-best method, i.e. the filtered model. We corroborate these results with a user study in Figure~\ref{fig:text-to-image} (b), where we generate images for randomly sampled captions in COCO. For each method, we generate 10 pairs of images and we ask 36 users to vote for the image with the best quality, and for the one with the most coherence to the prompt. Our method is overwhelmingly preferred to other methods. In particular, users prefer the image quality of our images 95\% of the cases and find our images better aligned with the prompts by 89\%. Notably, the user study reveals a well-established limitation in such models, i.e. FID vs CLIPScore tradeoffs do not necessarily correlate well with human perception, as shown also in SD-XL~\cite{podell2023sdxl}.

\paragraph{Quantitative results for class-conditional image generation.}
In Table~\ref{tab:class_coherence} (d), we compare to a baseline where we do not use the coherence score, and a filtered model, where we filter all samples with coherence scores lower than 0.5. When filtering, we observe on CIFAR that the model's performance dramatically drops. FIDs are worse than the baseline, showing that dropping images prevents generating high-quality images. CAD displays improved Accuracy over the baseline while having better image quality than the filtered baseline.

\paragraph{Qualitative results for semantic conditioning.}

In Figure \ref{fig:semantic_seg} (a), we present generated images derived from the same semantic map, obtained through the prediction of a segmentation network on an image from the ADE20K dataset~\cite{zhou2017scene}. Notably, we vary the prompts and coherence maps in our experiments.
When using a uniform coherence map set at $c=1$, the generated image aligns correctly with the semantic map but lacks semantic information, resulting in an image that may not appear meaningful. To introduce synthetic coherence maps, we employ Canny edge detection on the semantic map, creating regions of low coherence at class boundaries. This approach gives the model more flexibility in adjusting the shape of different objects, leading to a more realistic image (second column). 

Additionally, we manually edit the coherence map by introducing a low-coherence region in the form of a square. As depicted in the "$\sigma$+edit" columns of Figure~\ref{fig:semantic_seg} (a), the model tends to adhere to the shape defined by the coherence map. It strategically employs the location of this low-coherence region to incorporate objects that are specified in the prompt but absent from the semantic map. This observation is particularly highlighted in the final image in Figure~\ref{fig:semantic_seg} (a), where we overlay the coherence scores of two children obtained from another image onto the manipulated coherence map, while correspondingly adjusting the prompt. The model adeptly uses the degrees of freedom and shape information provided by the low-coherence region of the coherence map to seamlessly insert the children into the image.%

\paragraph{Quantitative results for semantic conditioning.}
In Figure~\ref{fig:semantic_seg} (b), we demonstrate that incorporating both the segmentation and the coherence map leads to a decrease in FID for both scenarios, with and without the text input, indicating the superior visual quality of the generated images. This behavior is expected as our model possesses greater freedom to generate realistic content instead of strictly adhering to the segmentation map.

\section{Conclusions}
We proposed a novel method for training conditional diffusion models with additional coherence information. By incorporating coherence scores into the conditioning process, our approach allows the model to dynamically adjust its reliance on the conditioning. We also extend the classifier-free guidance, enabling the derivation of conditional and unconditional models without the need for dropout during training. We have demonstrated that our method, called coherence-aware diffusion (CAD),  produces more diverse and realistic samples on various conditional generation tasks, including classification on CIFAR10, ImageNet and semantic segmentation on ADE20k.
\\
~\\
\noindent \textbf{Limitations.} The main limitation of CAD lies in the extraction of coherence scores, as unreliable coherence scores can lead to biases in the model. Future research includes focusing on more robust and reliable methods for obtaining coherence scores to further enhance the effectiveness and generalizability of our approach.

\section{Acknowledgments}
This work was supported by ANR project TOSAI ANR-20-IADJ- 0009, and was granted access to the HPC resources of IDRIS under the allocation 2023-AD011014246 made by GENCI. We would like to thank Vincent Lepetit, Romain Loiseau, Robin Courant, Mathis Petrovich, Teodor Poncu and the anonymous reviewers for their insightful comments and suggestion.
\newpage
\clearpage
{
    \small
    \bibliographystyle{ieeenat_fullname}
    \bibliography{main}
}
\newpage
\clearpage
\appendix

\section{CAD architecture}
In this section, we describe the building blocks of the proposed CAD architecture.

\subsection{RIN architecture for class conditional with coherence}
We first explain our adaptation of the RIN architecture~\cite{jabri2022scalable}, that we use in our experiments on class conditional generation in the context of coherence aware diffusion.
As shown in  Figure~\ref{fig:rin_block_arch}, the RIN block is composed of two branches: one with the latents and one with the patches. We concatenate the timestep, coherence and conditioning embeddings to the latents (olive, green and orange blocks), with the coherence embedding being an addition of our method to the original RIN architecture. Then, first, the latents gather information from the input patches via Cross-Attention. Second, the latents are processed with N self-attention layers. Finally, the patches are updated from the latents via Cross-Attention. The RIN architecture consists of stacking multiple RIN blocks, where the next RIN block receives the updated latents and patches. 

During inference, RIN takes as input a noisy version of the image, a class, a timestep, and a coherence token to predict the noise that has been added to the clean version of the image. To improve the sampling, output latents from a given step are forwarded as input to the next denoising step. For more details, see~\cite{jabri2022scalable}.
\begin{figure}
    \centering
    \includegraphics[width=\columnwidth]{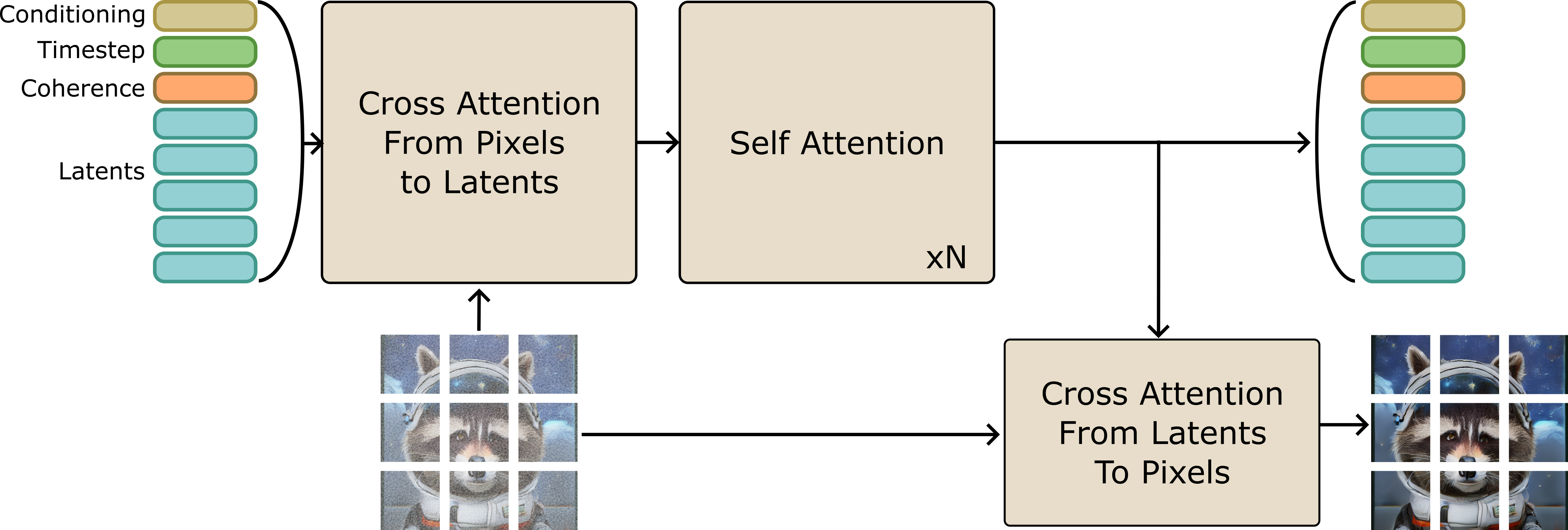}
    \caption{Architecture of the RIN Block modified to receive as input the coherence.}
    \label{fig:rin_block_arch}
\end{figure}

\section{Implementation details}
In this section, we present the implementation details for each experiment as well as the training setup.

\subsection{Text Conditining}

For the text conditional, we use the LAMB optimizer~\cite{you2019large} with a weight decay of 0.01. We use a learning rate of 0.001. The batch size is 1024.  We use a linear warmup of the learning rate for the first 10k steps and then use a cosine decay. We train all models for 300k steps. We use an EMA decay of 0.9999 for the last 50k steps. \\
We use the Stable diffusion VAE encoder~\cite{rombach2022high} and perform the diffusion process in its embedding space in which the image tokens have dimension 32x32x4.
We have 4 RIN blocks, each having 4 self-attention units. The data tokens dimension is 256 and the latent token dimension is 768. The input data is reduced to 256 data tokens by using a patch-size of 2.

\subsection{Class Conditioning}
For the class conditional experiments, we follow the hyper-parameters provided by the authors of RIN. We use the LAMB optimizer with weight decay of 0.01. We use a linear warmup of the learning rate for the first 10k steps and then use a cosine decay. We train all models for 150k steps. We use an EMA decay of 0.9999. \\
For CIFAR-10, we have 3 RIN blocks, each having 2 processing units. The data tokens dimension is 256 and the latent token dimension is 512. We use a patch-size of 2. We use a learning rate of 0.003. The batch-size is 256. \\
For ImageNet-64, we have 4 RIN blocks, each having 4 processing units. The data tokens dimension is 512 and the latent token dimension is 768. We use a patch-size of 8. We use a learning rate of 0.002. The batch-size is 1024.

\subsection{Semantic map conditioning}
\label{sup:sub:implementation_semantic}
Generations conditioned on semantic maps were obtained by training a ControlNet~\cite{zhang2023adding}.
To train the ControlNet, we created a dataset of images selected from the ADE20K~\cite{zhou2017scene}  and MS COCO~\cite{lin2014microsoft} datasets. This dataset contains tuples of the form (image, caption, semantic map, coherence map) that were generated from the original images.  The captions are obtained with BLIP2~\cite{li2023blip2}, an image captioning language model. We utilize a Maskformer~\cite{cheng2021maskformer} trained either on the MS COCO dataset or ADE20k to generate the semantic maps. To extract coherence values, we use the maximum class probability obtained from the softmax output of the Maskformer model. We employ a MaskFormer trained on ADE20K to generate the coherence map for MS COCO, and conversely, use a MaskFormer trained on MS COCO to obtain the coherence map for ADE20K. This approach helps mitigate the problem of overconfidence in predictions on the training set, reducing the tendency to have only high coherence scores across all pixels.
We used a batch size of 16, using the Adam optimizer with a learning rate of 1e-5.

\subsection{Computational cost}\label{section:GPU-hour}
Our method adds negligible training and inference time because we either modify existing architectures or add non-computationally expensive components.  Specifically, in terms of architecture, we are replacing one of the latent tokens with an embedded coherence score. For the text-conditional, we do add a new cross-attention layer, but most of the compute is still in the self-attention blocks.

In this project, we have used approximately 25,553 V100 hours for preliminary experiments including the CIFAR-10 experiments and 29,489 A100 hours for ImageNet and text-conditional experiments. Each GPU hour accounts for roughly 259 Wh for a total of 14,255kWh. 
For semantic segmentation, the training of the different ControlNets was performed using about 1,800 hours of Nvidia A100 GPUs in total, or about 470kWh. The training process required approximately 100 GPU hours for each model trained on ADE20k~\cite{zhou2017scene} and 200 GPU hours for MS COCO~\cite{lin2014microsoft}.

\section{Image from semantic map additional experiments}
\label{section:Segmentation}
In this section, we examine the effectiveness of incorporating coherence maps for semantic segmentation. We qualitatively and quantitatively compare the results of our CAD method against a baseline approach not using the coherence information.

\subsection{Quantitative results}
\label{sub:sem_quantitative}
Here, we quantitatively evaluate the effectiveness of incorporating coherence maps for semantic segmentation. For this, we experiment on two of the most popular datasets with segmentation: ADE20K~\cite{zhou2017scene} and MS COCO~\cite{lin2014microsoft}.  
To evaluate our results, we employ the Frechet Inception Distance (FID) and Inception Score (IS) for evaluating image quality. Additionally, Precision (P), Recall (R), Density (D), and Coverage (C) serve as manifold metrics, enabling an evaluation of the overlap between the generated and real image manifolds. Finally, we calculate the mean Intersection over Union (mIoU) by utilizing a pre-trained MaskFormer to predict a segmentation map from the generated image and compare it with the original semantic map. This helps illustrate the fidelity of the generated images to the ground truth.

\paragraph{Method comparison.} 
We compare the results of our CAD method with a baseline approach that excludes coherence information in both Table~\ref{tab:seg_quantitative} and Table~\ref{tab:seg_quantitative_coco}, with the complete results. We conduct experiments in two settings: with text (first two rows) and without text (last four rows). Additionally, we compare against two CAD variations. Similar to the binning strategy in text-to-image generation, ‘CAD bin’ encodes coherence into 5 equally distributed discrete bins. Furthermore, ‘CAD scalar’ utilizes a singular scalar coherence score for the entire image, equivalent to the mean of the original coherence map.

\begin{table}
  \begin{center}\captionof{table}{\textbf{Quantitative results on ADE20K when conditioning on semantic maps.} ‘CAD bin’ encodes the coherence into 5 equally distributed discrete bins. ‘CAD scalar’ uses a scalar coherence score for the whole image. CAD achieves better FID due to its enhanced ability to generate realistic objects in low coherence regions and superior mIoU as the leaked spatial information from the coherence map and the caption assist it to generate better samples (see Section~\ref{sub:sem_quantitative} for more details).}
        \resizebox{\linewidth}{!}{
        \begin{tabular}{c | ccccccc }%
                            & \multicolumn{7}{| c }{\cellcolor{lavenderblue}{\textbf{ADE20K}}}                                                          \\ \hline
        \textbf{ControlNet} & \textbf{FID}   & \textbf{IS}    & \textbf{mIoU}  & \textbf{P}      & \textbf{R}      & \textbf{D}       & \textbf{C}      \\ \hline
         Baseline           & 33.67          & \textbf{14.82} & 22.6           & 0.785           & 0.757           & 1.029            &  0.904          \\ 
         CAD                & \textbf{30.88} & 14.79          & \textbf{23.7}  & \textbf{0.844}  & \textbf{0.824}  & \textbf{1.0755}  &  \textbf{0.934} \\ \midrule
         Baseline w/o text  & 74.37          & 5.88           & 5.25           & 0.657           & 0.351           & 0.789            & 0.515           \\
         CAD w/o text       & 60.21          & 7.93           & 11.8           & 0.619           & 0.536           & 0.789            & 0.682           \\
         CAD bin            & 63.47          & 7.97           & 10.8           & 0.5875          & 0.4925          & 0.757            & 0.663           \\
         CAD scalar         & 74.69          & 6.15           & 3.11           & 0.6495          & 0.347           & 0.858            & 0.536           \\
       \end{tabular}
       }
       \label{tab:seg_quantitative}
   \end{center}
\end{table}

\begin{table}
\begin{center}
    \captionof{table}{\textbf{Quantitative results on MS COCO when conditioning on semantic maps.} ‘CAD bin’ encodes the coherence into 5 equally distributed discrete bins. ‘CAD scalar’ uses a scalar coherence score for the whole image (see Section~\ref{sub:sem_quantitative} for more details).}
    \resizebox{1.\linewidth}{!}{
    \begin{tabular}{c | ccccccc }
                         & \multicolumn{7}{| c }{\cellcolor{lavenderblue}{\textbf{COCO}}}                                                                  \\ \hline
     \textbf{ControlNet} & \textbf{FID}    & \textbf{IS}     & \textbf{mIoU}   & \textbf{P}      & \textbf{R}       & \textbf{D}       & \textbf{C}        \\ \hline
      Baseline           &  20.1           & \textbf{32.6}   & 35.1            & 0.7876          &  0.6760          & \textbf{1.0811}  &  0.8956           \\ 
      CAD                &  \textbf{18.1}  & 32.0            & \textbf{35.3}   & \textbf{0.8404} &  \textbf{0.8060} &  1.0687          &  \textbf{0.9304}  \\ \midrule
      Baseline w/o text  & 54.93           & 15.40           & 8.36            & 0.4884          & 0.4402           &  0.5297          &  0.5260           \\
      CAD w/o text       & 37.06           & 18.04           & 12.53           & 0.6222          & 0.6502           &  0.7599          &  0.7052           \\
      CAD bin            & 44.63           & 16.39           & 9.76            & 0.5636          & 0.5614           &  0.7075          &  0.6402           \\
      CAD scalar         & 55.82           & 15.92           & 8.10            & 0.5139          & 0.4382           &  0.5294          &  0.5341           \\
    \end{tabular}
    }
    \label{tab:seg_quantitative_coco}
\end{center}
\end{table}
\paragraph{Results.}
In Table~\ref{tab:seg_quantitative} and Table~\ref{tab:seg_quantitative_coco}, show the complete results of our method on ADE20k and COCO we demonstrate that using both the segmentation maps and the coherence maps lead to a decrease in FID for both scenarios, including or not the text input. This behavior is expected as our model possesses greater freedom to generate realistic content instead of sticking to the segmentation map uniquely (see \textit{e.g.}, the 4th column of Figure~\ref{fig:seg_qualitative_appendix}).
Furthermore, the improvement in the mIoU score can be attributed to two factors. First, when the input segmentation map is of low quality, the baseline method fails to capture important scene information. In contrast, our method benefits from additional information from the coherence map. Secondly, our method better leverages the caption in the low coherence region, mitigating the limitation of the segmentation map's limited number of classes (as seen in 4th row in Figure~\ref{fig:prompt_var}, there is no ping-pong table class in the COCO dataset). 

\subsection{Additional Visualizations}
We show additional results in Figure~\ref{fig:seg_qualitative_appendix}, where, from the left column to the right one, we highlight the segmentation input, the coherence map, the image generated by the baseline, the image generated by our methods and the reference image. The coherence map reveals spatial/shape details of the scene. For instance, when comparing our method to a ControlNet trained solely with the segmentation map, our approach, which incorporates both segmentation and coherence, accurately reconstructs the curtain's shape and the windows in the first row, or reconstructs a cloud in the back of the plane in the second row. Moreover, the efficacy of our method becomes even more apparent when the segmentation map is of poor quality and the coherence score is low. For instance, as shown in the third row, the basic ControlNet attempts to adhere to the limited information provided by the flawed segmentation map, resulting in a scene with multiple arms displayed (third column). In contrast, our approach benefits from the flexibility given by the coherence map, allowing for more consistent image generation. Interestingly, our method also exhibits localized self-correction, as shown on the fourth row. In particular, our model is refraining from generating the hand region due to its low coherence in the input. Finally, we demonstrate the model's responsiveness to textual input in the last row. We present in the last row an example where the original image does not contain snow, but have high coherence in the sky. Both models generate some snow based on the caption, but in line with the coherence map, our model does not generate snow in the sky, thus being closer to the real image.

\begin{figure*}
    \begin{minipage}[t]{\textwidth}
        \centering
        \subcaption*{Caption: An hotel room with a bed, chair and table:}
        \includegraphics[width=\textwidth]{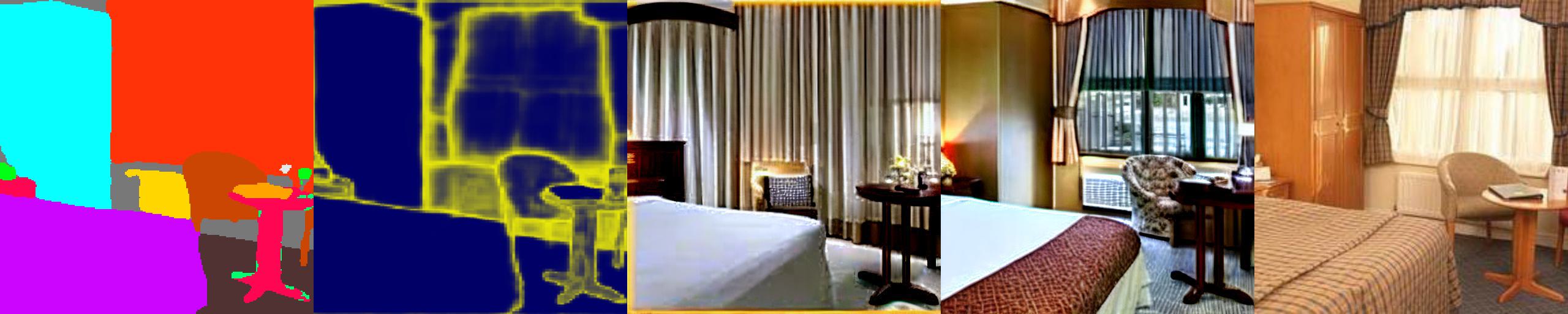}
    \end{minipage}
    \\
    \begin{minipage}[t]{\textwidth}
        \centering
        \subcaption*{Caption: A white airplane on the runway:}
        \includegraphics[width=\textwidth]{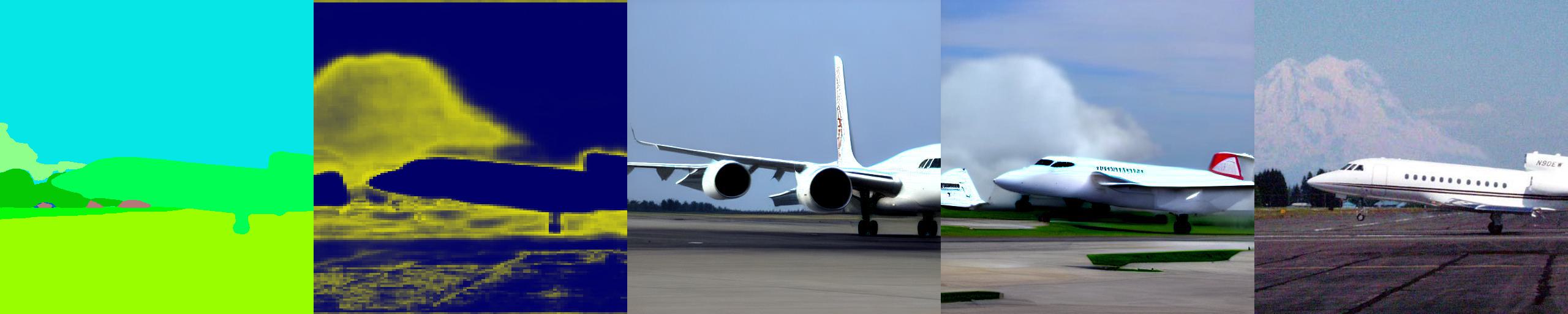}    
    \end{minipage}
    \\
    \begin{minipage}[t]{\textwidth}
        \centering
        \subcaption*{Caption: A man sitting in front of a slot machine:}
        \includegraphics[width=\textwidth]{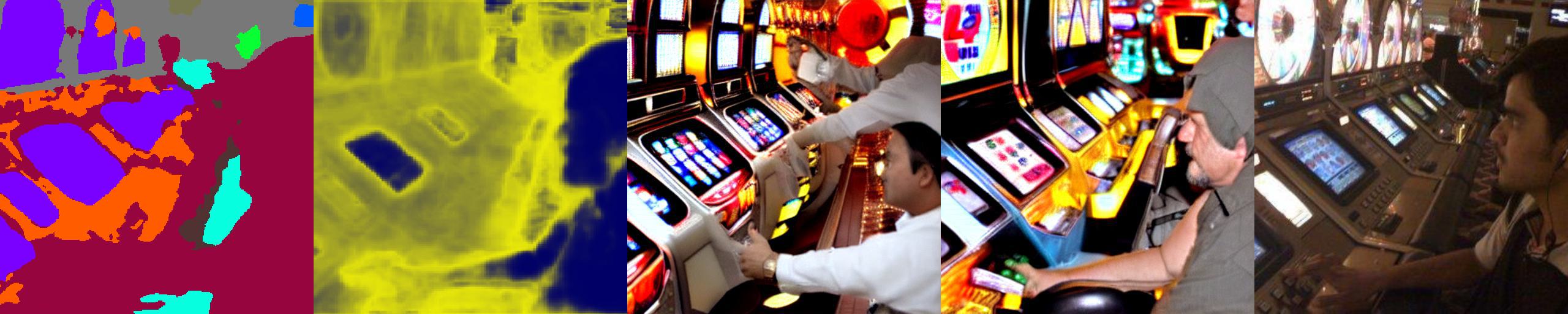}    
    \end{minipage}
    \\
    \begin{minipage}[t]{\textwidth}
        \centering
        \subcaption*{Caption: A large room with ping-pong tables and people playing:}
        \includegraphics[width=\textwidth]{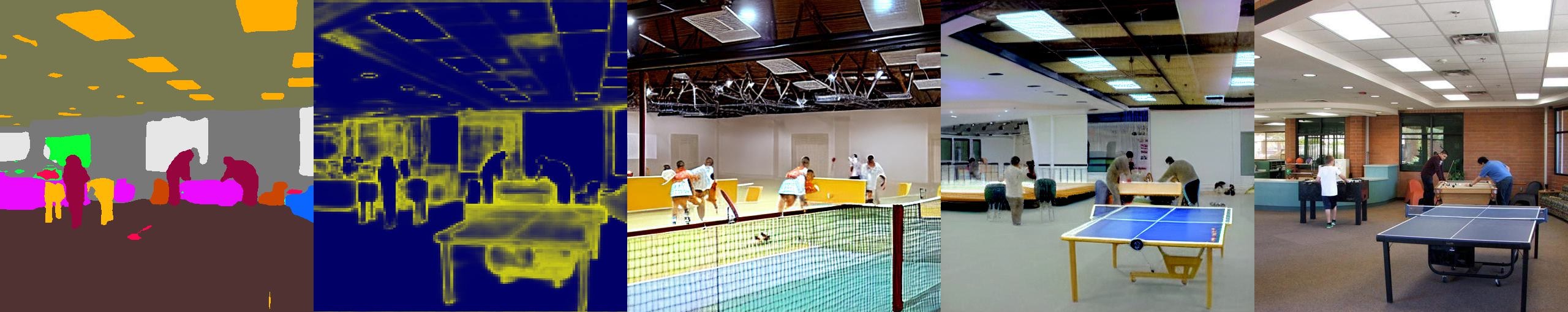}    
    \end{minipage}
    \\
    \begin{minipage}[t]{\textwidth}
        \centering
        \subcaption*{Caption: A drilling rig in the middle of a snowy field:}
        \includegraphics[width=\textwidth]{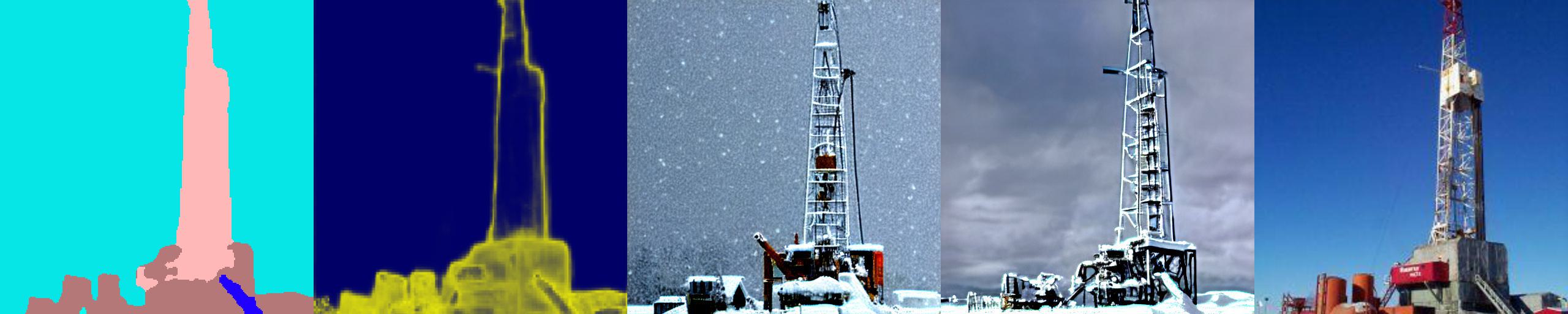}    
    \end{minipage}
    \\
    \\
    \put(0,0){\hspace{6mm} \textbf{Segmentation} \hspace{15mm} \textbf{Coherence} \hspace{18mm} \textbf{Baseline} \hspace{23mm} \textbf{Ours} \hspace{22mm} \textbf{Real Image}\hfill}
    \caption{\textbf{Qualitative results on ADE20K:} Examples of images generated conditionally to text and semantic maps.}
    \label{fig:seg_qualitative_appendix}
\end{figure*}

\subsection{Prompt generalization}
In this subsection, we demonstrate the sensitivity of our method to the caption input. We observe in Figure~\ref{fig:prompt_var} that our CAD method can successfully generalize to different types of captions, such as `dining table', `billiard', or `ping pong table' and adjust the scene accordingly. Moreover, even when the table is not explicitly mentioned in the caption (as seen on the rightmost side of the figure), our method exhibits strong generalization capabilities and successfully generates the table. 

\begin{figure*}
    \centering
    \includegraphics[trim=0cm 3.5cm 0cm 3cm, width=\textwidth]{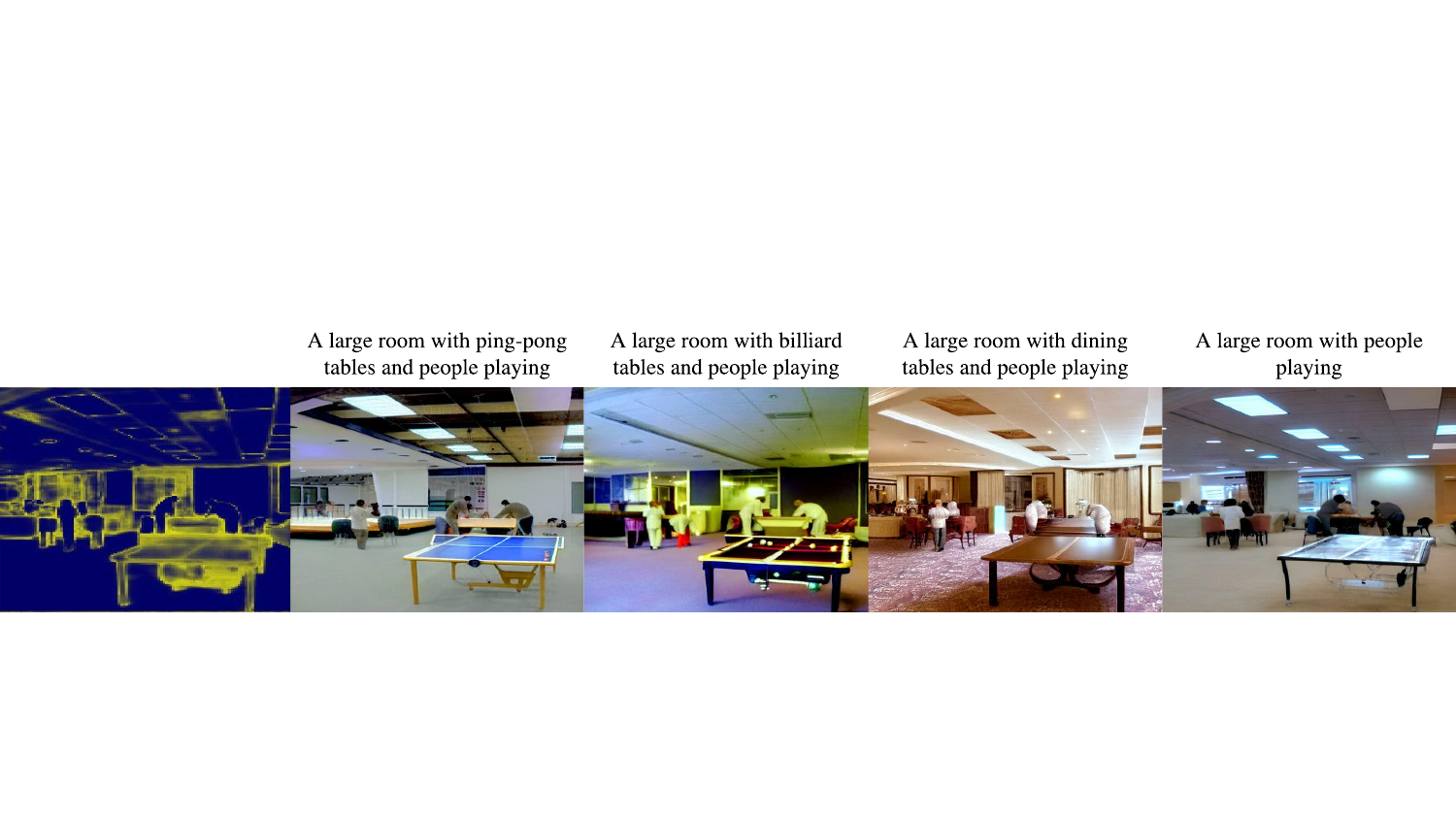}
    \caption{\textbf{Caption generalization:} Our methods demonstrate remarkable capability in leveraging the coherence map to generalize to diverse prompt inputs.}
    \label{fig:prompt_var}
\end{figure*}

\subsection{Coherence Interpolation}
In Figure~\ref{fig:coherence_interpolation}, we demonstrate the significance of the coherence map in the conditioning of the ControlNet. In this experiment, we make an interpolation of the coherence map from the maximum coherence score everywhere (left) to very low coherence (right). When the input has high coherence throughout, the generated image lacks the presence of a ping pong table as it is not present in the semantic map. However, as the coherence score decreases, our methods recognize the shape of the ping pong table and successfully generates it in the image. It is worth noting that even at a low scale of the coherence score (second column corresponds to 1e-4 times the original value), our method is still able to reconstruct the table, even if it is not in the segmentation input. When we artificially reduce the coherence (two times less coherence, in the last column), our method is still able to generate a consistent scene without any artifacts.

\begin{figure*}
    \centering
    \includegraphics[trim=0cm 2.5cm 0cm 0cm, width=\textwidth]{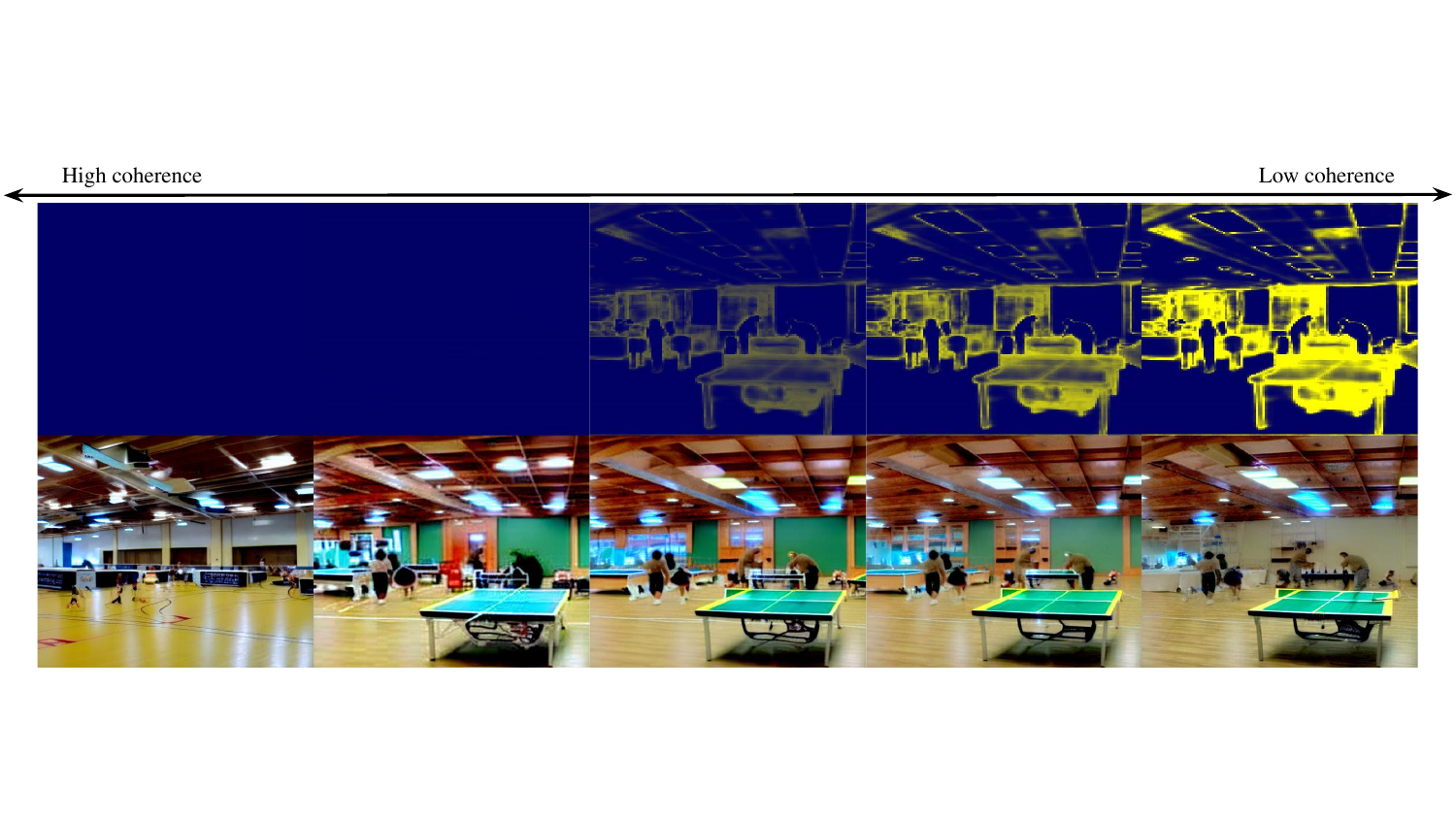}
    \caption{\textbf{Coherence interpolation:} In the first column, we artificially provide our ControlNet with a coherence map having the maximum value everywhere. Indeed, our models do not generate the ping pong tables. But, as soon as we decrease the coherence toward its original value, the ping pong tables start to appear. Finally, in the last column, we provide the model with a coherence map that is half as confident as the original value and demonstrate that we can generate an image without artifacts}
    \label{fig:coherence_interpolation}
\end{figure*}

\section{Class conditional experiments}
\subsection{Simulating annotation noise}
Our approach for class-conditional image generation relies on the assumption that the dataset comes with annotated coherence scores.
However, such scores are not always available for traditional image-generation datasets.
To address this issue, we propose to simulate annotation noise by re-sampling from a dataset with clean annotations.

We associate an error probability $\alpha$ with each label. We assume that when the annotator is wrong, they misclassify uniformly over all classes, where $N$ is the total number of classes. \
This leads to the following model:
\begin{align}
    \bar{y} \sim p_{y, \alpha} \text{   with    } 
    p_{y, \alpha}(\bar{Y} = k) =  \begin{cases}
        1-\alpha & \text{if } y=k, \\
        \frac{\alpha}{N-1} & \text{otherwise}. \\
    \end{cases}
\end{align}
We also define a strategy to remap the entire dataset using a normalized entropy-based coherence measure. We use the normalized entropy so that 0 maps to no coherence at all and 1 to total coherence in the label.
We define the following coherence function:
\begin{align}
\nonumber    E(\alpha) = &\frac{-1}{\log(N)}\left((1-\alpha)\log(1-\alpha) + \alpha \log\left(\frac{\alpha}{N-1}\right)\right) \\ 
    &\text{ for } \alpha \in \left[0,\frac{N-1}{N}\right].
\end{align}

To ensure that the dataset has samples with varying levels of confidence, we define a target entropy cumulative distribution. To achieve this, we use a piecewise-linear function:
\begin{equation}
    E_{\beta, th}(t) = \begin{cases}
                t\frac{\beta}{\kappa} & \text{if } t<\kappa, \\
                1+(t-1)\frac{1-\beta}{1-\kappa} & \text{otherwise}. \\
             \end{cases}
\end{equation}
where $\kappa$ represents a threshold and $\beta$ represents the entropy at this threshold. This function construction defines a low entropy region before the threshold and a high entropy region after the threshold.
\\
Finally, for each sample in the dataset $(X,y)$, we sample $t \in \mathcal{U}[0,1]$, and associate a target entropy $u$. We then compute the associated error probability $\alpha = E^{-1}(u)$, and resample according to $p_{y, \alpha}$ to obtain the tuple $(X, \bar{y}, 1-u)$\footnote{The coherence goes in the opposite direction of the entropy.}. This process allows us to generate synthetic data points with varying degrees of annotation noise and coherence.

\subsection{Quantitative Results}
In this section we add more results on ImageNet with different levels of noise. We observe in Table~\ref{tab:quant_class} results on ImageNet for $\beta\in\{0.2, 0.5, 0.8\}$. We first observe the less coherent the labels are the worse the results get in both image quality (FID) but also in accuracy (Acc). However, our method CAD manages to achieve better results than other methods in the context of un-coherent labels. This is further amplified when leveraging coherence aware guidance
{
\renewcommand{\arraystretch}{1.4}
\begin{table}[h]
\centering
\caption{Quantitative results for class-conditional image generation. Our coherence aware diffusion (CAD) is compared to a baseline model and a training set filtering strategy for different levels of label noise $\beta$. We show that CAD achieves higher fidelity and better accuracy.}
\resizebox{\columnwidth}{!}{
\begin{tabular}{cc | ccccccc }
\cline{3-9}
 & & \multicolumn{7}{ c }{\cellcolor{lavenderblue}{\textbf{ImageNet}}}  \\
 \hline
 \textbf{$\beta$} & \textbf{Method} & \textbf{FID} & \textbf{IS} & \textbf{Acc} & \textbf{P} & \textbf{R} & \textbf{D} & \textbf{C} \\
\hline
 & Conditional & 7.56 & 34.26 & 0.475 & 0.595 & 0.610 & 0.768 & 0.706 \\
\hline
\multirow{4}{*}{0.2} & Baseline & 11.09 & 23.54 & 0.264 & 0.562 & 0.598 & 0.670 & 0.594\\
 & Filtered & 8.53 & 29.27 & 0.389 & 0.591 & 0.609 & 0.756 & 0.688\\
 & CAD & 8.17 & 27.75 & 0.367 & 0.585 & 0.615 & 0.736 & 0.658\\
 & CA-CFG $\omega=1$ & 5.95 & 68.95 & 0.679 & 0.742 & 0.477 & 1.203 & 0.812 \\
 \hline
\multirow{4}{*}{0.5} & Baseline & 14.38 & 20.46 & 0.168 & 0.539 & 0.579 & 0.595 & 0.505 \\
 & Filtered & 10.20 & 26.59 & 0.338 & 0.573 & 0.608 & 0.707 & 0.645 \\
 & CAD & 9.11 & 25.97 & 0.327 & 0.571 & 0.610 & 0.714 & 0.633\\
 & CA-CFG $\omega=1$ & 5.95 & 68.95 & 0.679 & 0.742 & 0.477 & 1.203 & 0.812 \\
 \hline
\multirow{4}{*}{0.8} & Baseline & 20.10 & 17.28 & 0.100 & 0.502 & 0.535 & 0.526 & 0.417 \\
 & Filtered & 12.00 & 24.55 & 0.292 & 0.420 & 0.712 & 0.647 & 0.605\\
 & CAD & 11.39 & 22.08 & 0.248 & 0.558 & 0.590 & 0.682 & 0.574\\
 & CA-CFG $\omega=1$ & 6.70 & 44.27 & 0.523 & 0.695 & 0.483 & 1.035 & 0.743
\end{tabular}
}
\label{tab:quant_class}
\end{table}
}

\section{Theoretical analysis}
In this section, we motivate the use of coherence as an additional conditioning for diffusion models. Under assumptions that are verified empirically, we show that coherence aware diffusion can transition from an unconditional model to a conditional model simply by varying the coherence passed to the model.
First, we define a consistency property of the coherence embedding as follows:
\begin{definition}
    We denote \emph{coherence consistent} a conditional embedding $h(y, c)$ of the condition $y\in \mathcal{Y}$ under coherence $c\in [0, 1]$, if $\forall y_1, y_2 \in \mathcal{Y}$ we have
    \begin{align}
        \lim_{c \rightarrow 0} \|h(y_1, c) - h(y_2, c)\| = 0~.
    \end{align}
\end{definition}
In other words, an embedding is \emph{coherence consistent} if it tends to produce the same vector as the coherence approaches 0.
This property is a sufficient condition to constrain the behavior of the diffusion model. Indeed, the following proposition is easily derived from it:
\begin{proposition}
    Lipschitz continuous conditional neural diffusion models that leverage \emph{coherence consistent} embeddings for the conditioning are equivalent to unconditional models at low coherence.
\end{proposition}
\begin{proof}
    We have to prove the following equivalent statement:
    Let $\epsilon_\theta: x_t, t, h(y, c) \mapsto \hat\epsilon_t$ be a Lipschitz continuous neural diffusion model that predicts the noise $\hat\epsilon_t$ at time $t$ from the noisy sample $x_t$ with the help of the condition $y$ embedded using the \emph{coherence consistent} embedding $h$ under coherence $c$. Then, $\forall \eta > 0$ and $\forall x_t, t, y_1 \neq y_2$, there exists $C > 0$ such that for all $0 < c \leq C$, we have
    \begin{align}
        \| \epsilon_\theta(x_t, t, h(y_1, c)) - \epsilon_\theta(x_t, t, h(y_2, c)) \|^2 < \eta~.
    \end{align}
    
    By Lipschitz property of $\epsilon_\theta$, we have $
        \| \epsilon_\theta(x_t, t, h(y_1, c)) - \epsilon_\theta(x_t, t, h(y_2, c)) \|^2 \leq L^2 \|h(y_1, c) - h(y_2, c) \|^2$.
    From the \emph{coherence consistent} property, there exists $C > 0$ such that for all $0< c \leq C, \|h(y_1, c) - h(y_2, c) \| < \sqrt{\eta}/L$.
\end{proof}
The following contrapositive necessary condition on the coherence directly follows from this proposition:
\begin{corollary}
    Lipschitz continuous conditional neural diffusion models that leverage \emph{coherence aware} embeddings require high coherence to behave like conditional models.
\end{corollary}
In practice, we show in the experiments that the coherence consistency property tends to naturally emerge during training and that consequently coherence aware diffusion provides a tunable prompt parameter to sample from unconditional to conditional models.

\section{Additional Qualitative Results}
In this section, we provide additional samples from our method. Most of the prompts are sampled from the Lexica.art website.

\def\captionsquantjoin{0/An old-world galleon navigating through turbulent ocean waves under a stormy sky\, lit by flashes of lightning, 1/an oil painting of rain at a traditional Chinese town, 2/portrait photo of a asia old warrior chief\, tribal panther make up\, blue on red\, side profile\, looking away\, serious eyes\, 50mm portrait photography\, hard rim lighting photography,3/a blue jay stops on the top of a helmet of Japanese samurai\, background with sakura tree, 4/A cute little matte low poly isometric cherry blossom forest island\, waterfalls\, lighting\, soft shadows\, trending on Artstation\, 3d render\, monument valley\, fez video game., 5/Underwater cathedral, 6/A cozy gingerbread house nestled in a dusting of powdered sugar snow\, adorned with vibrant candy canes and shimmering gumdrops, 7/a teddy bear wearing blue ribbon taking selfie in a small boat in the center of a lake, 8/Pirate ship trapped in a cosmic maelstrom nebula\, rendered in cosmic beach whirlpool engine\, volumetric lighting\, spectacular\, ambient lights\, light pollution\, cinematic atmosphere\, art nouveau style\, illustration art artwork by SenseiJaye\, intricate detail.}
\begin{figure*}[h]
\caption{Samples from our CAD-B model at 512 resolution with associated caption}
\centering
\foreach \i/\captionquant in  \captionsquantjoin{
  \begin{minipage}[t]{0.30\textwidth}
  \includegraphics[width=\textwidth]{images/supp_mat_samples_with_text/\i.jpg}
  \caption*{\footnotesize{\captionquant}}
  \end{minipage}
}

\end{figure*}

\begin{figure*}[h]
\centering
\foreach \i in {1,...,12} {
  \begin{minipage}{0.32\textwidth}
  \includegraphics[width=\textwidth]{images/supp_mat_samples/image_\i.jpg}
  \end{minipage}
}
\caption{Samples from our CAD-B model at 512 resolution}
\end{figure*}

\begin{figure*}[h]
\centering
\foreach \i in {13,...,24} {
  \begin{minipage}{0.32\textwidth}
  \includegraphics[width=\textwidth]{images/supp_mat_samples/image_\i.jpg}
  \end{minipage}
}
\caption{Samples from our CAD-B model at 512 resolution}
\end{figure*}

\begin{figure*}[h]
\centering
\foreach \i in {25,...,36} {
  \begin{minipage}{0.32\textwidth}
  \includegraphics[width=\textwidth]{images/supp_mat_samples/image_\i.jpg}
  \end{minipage}
}
\caption{Samples from our CAD-B model at 512 resolution}
\end{figure*}

\begin{figure*}[h]
\centering
\foreach \i in {37,...,48} {
  \begin{minipage}{0.32\textwidth}
  \includegraphics[width=\textwidth]{images/supp_mat_samples/image_\i.jpg}
  \end{minipage}
}
\caption{Samples from our CAD-B model at 512 resolution}
\end{figure*}

\end{document}